\documentclass[lettersize,journal]{IEEEtran}
\usepackage{amsmath,amsfonts}
\IEEEoverridecommandlockouts
\normalsize

\usepackage{cite}
\usepackage{url}
\usepackage{hyperref}
\hypersetup{colorlinks=true, linkcolor=red, citecolor=purple, urlcolor=blue} 
\usepackage{amssymb}
\usepackage{amsmath}
\usepackage{array}
\usepackage{amsthm}
\allowdisplaybreaks 
\usepackage{autobreak}
\usepackage{mathptmx}
\usepackage{mathtools, cuted}
\usepackage{algpseudocode}
\usepackage[ruled, vlined, linesnumbered]{algorithm2e}
\usepackage{graphicx}
\usepackage{subcaption} 
\usepackage{bbm}
\usepackage{booktabs}
\usepackage{multirow}
\usepackage{makecell}
\usepackage[table,xcdraw,dvipsnames]{xcolor}
\usepackage{colortbl}

\usepackage[left=1.62cm,right=1.68cm,top=1.9cm]{geometry}
\allowdisplaybreaks

\newtheorem{Theorem}{Theorem}
\newtheorem{Lemma}{Lemma}

\newtheorem{Assumption}{Assumption}
\newtheorem{Remark}{Remark}

\newcommand{\rs}{\!\!}
\newcolumntype{C}[1]{>{\centering \arraybackslash}p{#1}}
\newcommand{\bblue}{\textcolor{black}}

\usepackage{acronym}
\acrodef{ml}[ML]{machine learning}
\acrodef{fl}[FL]{federated learning}
\acrodef{hfl}[HFL]{hierarchical federated learning}
\acrodef{hsfl}[{\tt HSFL}]{hierarchical split federated learning}
\acrodef{phsfl}[{\tt PHSFL}]{personalized hierarchical split federated learning}

\acrodef{rl}[RL]{reinforcement learning}
\acrodef{drl}[DRL]{deep reinforcement learning}
\acrodef{cs}[CS]{central server}
\acrodef{sl}[SL]{split learning}
\acrodef{sfl}[SFL]{split federated learning}

\acrodef{bs}[BS]{base station}
\acrodef{isp}[ISP]{(wireless) internet service provider}
\acrodef{ue}[UE]{user equipment}
\acrodef{es}[ES]{edge server}
\acrodef{csp}[CSP]{content service provider}
\acrodef{fedavg}[FedAvg]{federated averaging}
\acrodef{fednova}[FedNova]{federated normalized averaging}
\acrodef{afa}[AFA]{anarchic federated averaging}
\acrodef{feddisco}[FedDisco]{federated learning with discrepancy-aware collaboration}
\acrodef{scaffold}[SCAFFOLD]{stochastic controlled averaging algorithm}
\acrodef{osafl}[OSAFL]{\underline{\textbf{o}}nline-\underline{\textbf{s}}core-\underline{\textbf{a}}ided \underline{\textbf{f}}ederated \underline{\textbf{l}}earning}
\acrodef{iot}[IoT]{Internet of Things}
\acrodef{os}[OS]{operating system}
\acrodef{iid}[IID]{independent and identically distributed}

\acrodef{sgd}[SGD]{stochastic gradient descent}
\acrodef{cpu}[CPU]{central processing unit}
\acrodef{gpu}[GPU]{graphics processing unit}
\acrodef{prb}[pRB]{physical resource block}
\acrodef{snr}[SNR]{signal-to-noise-ratio}
\acrodef{lp}[LP]{linear programming}
\acrodef{fpp}[FPP]{floating point precision}
\acrodef{cdf}[CDF]{cumulative distribution function}
\acrodef{uav}[UAV]{unmanned aerial vehicles}
\acrodef{ap}[AP]{access point}
\acrodef{hz}[Hz]{hertz}
\acrodef{csi}[CSI]{channel state information}
\acrodef{kkt}[KKT]{Karush–Kuhn–Tucker}
\acrodef{fifo}[{\tt FIFO}]{first-in-first-out}
\acrodef{trimtoplabel}[{\tt TrimTopLabel}]{trim top label}
\acrodef{fcn}[{\tt FCN}]{fully connected neural network}
\acrodef{lstm}[{\tt LSTM}]{long short-term memory}
\acrodef{cnn}[{\tt CNN}]{convolutional neural network}
\acrodef{fc}[{\tt FC}]{fully connected}
\acrodef{isac}[ISAC]{integrated sensing and communication}

\title{Personalized Hierarchical Split Federated Learning in Wireless Networks} 
\author{Md Ferdous Pervej and Andreas F. Molisch \\
\IEEEauthorblockA{Ming Hsieh Department of ECE, University of Southern California, Los Angeles, CA 90089, USA \\
Emails: {\tt \{pervej, molisch\}@usc.edu} }
\thanks{This work was supported by NSF-IITP Project $2152646$.}
\thanks{\copyright $2025$ IEEE. Personal use of this material is permitted. Permission from IEEE must be obtained for all other uses, in any current or future media, including reprinting/republishing this material for advertising or promotional purposes, creating new collective works, for resale or redistribution to servers or lists, or reuse of any copyrighted component of this work in other works.}
\vspace{-0.35in}
}

\begin{document}

\maketitle
\IEEEpeerreviewmaketitle

\begin{abstract}
Extreme resource constraints make large-scale \ac{ml} with distributed clients challenging in wireless networks. 
On the one hand, large-scale \ac{ml} requires massive information exchange between clients and server(s). 
On the other hand, these clients have limited battery and computation powers that are often dedicated to operational computations.
\Ac{sfl} is emerging as a potential solution to mitigate these challenges, by splitting the \ac{ml} model into client-side and server-side model blocks, where only the client-side block is trained on the client device.
However, practical applications require personalized models that are suitable for the client's personal task.
Motivated by this, we propose a \ac{phsfl} algorithm that is specially designed to achieve better personalization performance.
More specially, owing to the fact that regardless of the severity of the statistical data distributions across the clients, many of the features have similar attributes, we only train the body part of the \ac{fl} model while keeping the (randomly initialized) classifier frozen during the training phase.
We first perform extensive theoretical analysis to understand the impact of model splitting and hierarchical model aggregations on the global model.
Once the global model is trained, we fine-tune each client classifier to obtain the personalized models.
Our empirical findings suggest that while the globally trained model with the untrained classifier performs quite similarly to other existing solutions, the fine-tuned models show significantly improved personalized performance.

\end{abstract}


\vspace{-0.15in}

\section{Introduction}
\acresetall 
\noindent 
Given the massive number of wireless devices that are packed with onboard computation chips, we are one step closer to a connected world. 
While these devices perform many onboard computations, they usually have limited computational and storage resources that can be dedicated to training \ac{ml} models. Conversely, cloud computing for \ac{ml} models raises severe privacy questions. 
Among various distributed learning approaches, \ac{fl} \cite{mcmahan2017communication} is widely popular as it lets the devices keep their data private.
These distributed learning algorithms are not confined to theory anymore; we have seen their practical usage in many real-world applications, such as the Google keyboard (Gboard) \cite{hard2018federated}.

\Ac{fl}, however, has its own challenges \cite{kairouz2021advances}, which are mostly the results of diverse system configurations, commonly known as system heterogeneity, and statistical data distributions of the client devices.
On top of these common issues, varying wireless network conditions also largely affect the \ac{fl} training process when the devices are wireless: the model has to be exchanged using the wireless channel between the devices and server(s).
While such difficulties are often addressed jointly by optimizing the networking and computational resources (e.g., see \cite{pervej2023resource,chen2020joint} and the references therein), traditional \ac{fl} may not be applied directly in many practical resource-constrained applications if the end devices need to train the entire model \cite{lin2024split}.

\Ac{sl} \cite{vepakomma2018split} brings a potential solution to the limited-resource constraints problem by dividing the model into two parts: (a) a much smaller \emph{front-end} part and (b) a bulky \emph{back-end} part.
The front-end part --- also called the \emph{client-side} model --- is trained on the user device, while the bulky back-end part --- also called the \emph{server-side} model --- is trained on the server. 
\Ac{sl} thus can enable training extremely bulky models at the wireless network edge, which typically incurs significant computational and communication overheads in traditional \ac{fl} (e.g., see \cite{pervej2024hierarchical} and the references therein). 
For example, as large foundation models \cite{vaswani2017attention} --- that can have billions of trainable model parameters --- are becoming a part of our day-to-day life and are also envisioned to be an integral part of wireless networks \cite{chen2024big}, \ac{sl} can facilitate training these large models at the wireless edge.

While \ac{sl} can be integrated into the \ac{fl} framework  \cite{xu2024accelerating}, it still has several key challenges, particularly when the clients' data distributions are highly non-IID (independent and identically distributed). 
While \ac{fl} typically seeks a single global model that can be used for all users, the performance reduces drastically under severe non-IID data distributions.
This is due to the fact that general \ac{fl} algorithms, like the widely popular \ac{fedavg} \cite{mcmahan2017communication}, may inherently push the global model toward the local model weights of a client who has more training samples: these samples, however, may not be statistically significant. 
A such-trained global model underperforms in other clients' test datasets, raising severe concerns about this ``\emph{one model fits all}" approach.
To empirically illustrate this, we implemented a simple \ac{hsfl} algorithm with $100$ clients and $4$ edge servers that follows the general architecture of \cite{khan2024joint} and performs $5$ local epochs $3$ edge rounds and $100$ global rounds.
The test performances, as shown in Fig. \ref{hsfl_General_Results}, show that the globally trained model yields very different test accuracies in different clients' test datasets.

\begin{figure}[!t]
    \centering
    \includegraphics[trim={60 30 75 60},clip, width=0.98\linewidth, height=0.24\textheight]{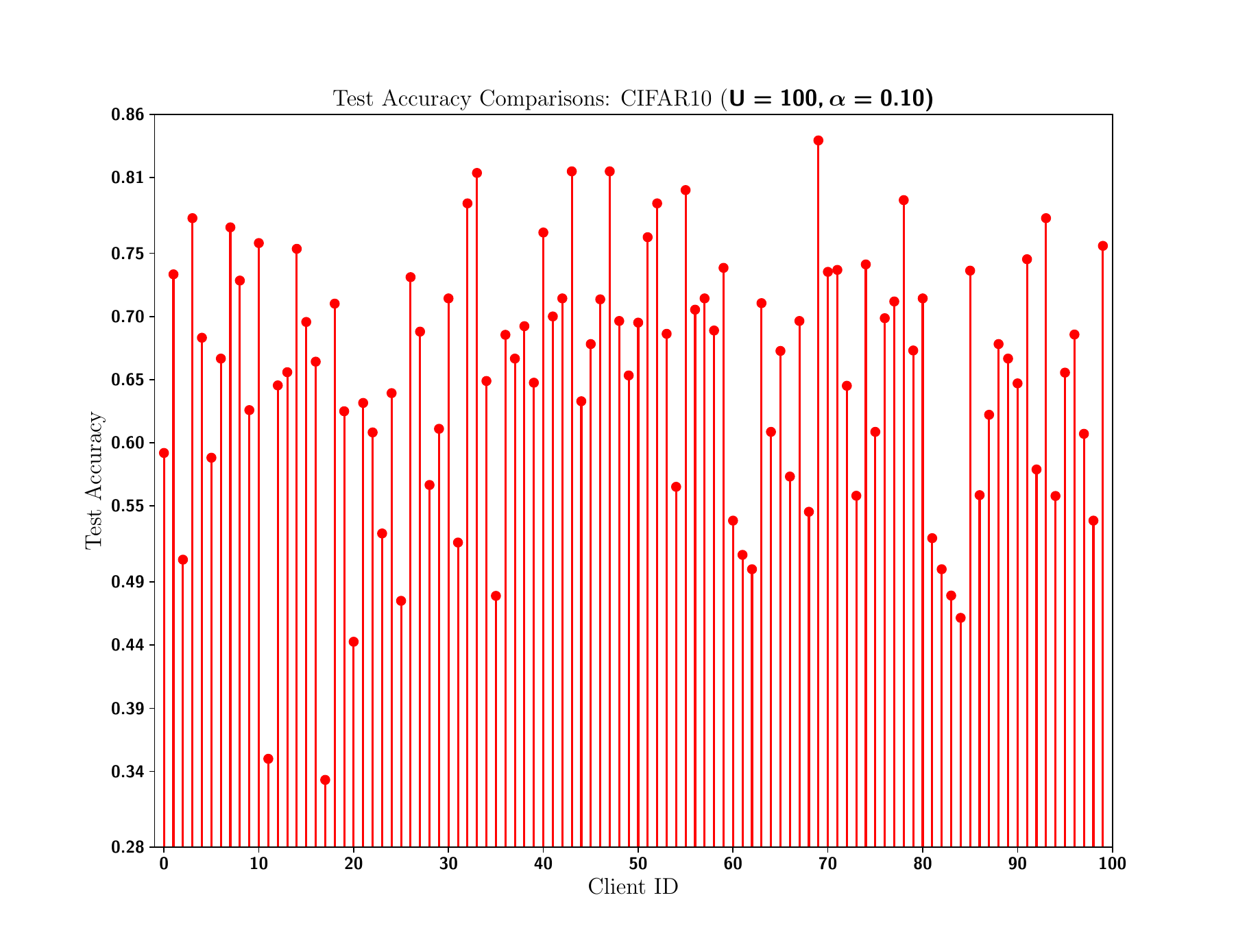}
    \caption{Globally trained model's performance on CIFAR$10$: $65.36\%$, $83.93\%$ and $33.33\%$ mean, maximum and minimum test accuracy, respectively, across $100$ clients, when data samples are distributed following $\mathrm{Dir}(\pmb{\alpha}=\mathbf{0.1})$ \cite{pervej2023resource}}
    \label{hsfl_General_Results}
\end{figure}

\subsection{State of the Art}
\noindent
Many recent works extended the idea of \ac{sl} \cite{vepakomma2018split} into variants of \ac{fl} \cite{mcmahan2017communication} algorithms \cite{xu2024accelerating,liu2022wireless,khan2024joint,liao2024parallelsfl,xia2022hsfl,ao2024federated,lin2024efficient}.
Xu \textit{et al.} proposed a \ac{sfl} algorithm leveraging a single server with distributed clients \cite{xu2024accelerating}.
In particular, the authors assumed that the client-side model is aggregated only after finishing the local rounds, while the server can aggregate the server-side models in each local training round. 
Liu \textit{et al.} proposed hybrid \ac{sfl}, where a part of the clients train their entire models locally, while others collaborate with their serving \ac{bs} to train their respective model following \ac{sl} \cite{liu2022wireless}.
Ao \textit{et al.} proposed \ac{sfl} assuming the \ac{bs} selects a subset of the clients to participate in model training \cite{ao2024federated}. 
In particular, \cite{ao2024federated} jointly optimizes client scheduling, power allocation, and cut layer selection to minimize a weighted utility function that strikes a balance between the training time and energy overheads.

Khan \textit{et al.} proposed an \ac{hsfl} algorithm leveraging \ac{hfl} and \ac{sl} \cite{khan2024joint}.
In particular, the authors aimed at optimizing the latency and energy cost associated with \ac{hsfl}. 
Liao \textit{et al.} proposed a similar \ac{hsfl} algorithm that partitioned clients into different clusters \cite{liao2024parallelsfl}.
The authors leveraged grouping the clients into appropriate clusters to tame the system heterogeneity.
However, statistical data heterogeneity is well known to cause client drifts \cite{karimireddy2020scaffold}.
Lin \textit{et al.} proposed a parallel \ac{sl} algorithm that jointly optimized the cut layer splitting strategy, radio allocation, and power control to minimize per-round training latency \cite{lin2024efficient}.
As opposed to \ac{sfl}, parallel \ac{sl} did not aggregate client-side models.

On the personalized \ac{sfl} side, Han \textit{et al.} proposed weighted aggregation of the local and global models during the local model synchronization phase \cite{han2023splitgp}. 
Similar ideas were also explored in \cite{sun2021partialfed}, where the authors mixed partial global model parameters with the local model parameters during the local model synchronization.
However, \cite{sun2021partialfed} did not explore \ac{sl}.
Chen \textit{et al.} proposed a 3-stage U-shape split learning algorithm \cite{chen2023personalized}.
More specifically, the authors divided the model into front, middle and back parts, where the clients retained the front and back parts, while the server had the bulky middle part. 
Such U-shape architecture incurs additional communication burden compared to \cite{han2023splitgp}.

\subsection{Research Gaps and Our Contributions}
\noindent
The existing studies \cite{xu2024accelerating,liu2022wireless,khan2024joint,liao2024parallelsfl,xia2022hsfl,ao2024federated,lin2024efficient} considered \ac{sfl}, \ac{hsfl} and parallel \ac{sl} extensively without addressing the need for personalization.  
While \cite{han2023splitgp,chen2023personalized} addressed joint personalization and split \ac{fl}, these works were based on a traditional single server with distributed clients case.
Therefore, weighted aggregation in multi-tier/hierarchical networks may lose personalization ability if the learning rate is not significantly low.
Moreover, training the entire model is proven to have poor personalization capability \cite{oh2022fedbabu}: even though the model was split into client-side and server-side parts, all model blocks on both sides of the models were updated in \cite{han2023splitgp,chen2023personalized}.

Motivated by the above facts, we propose a \ac{phsfl} algorithm that integrates \ac{hfl} and \ac{sl}. 
The designed algorithm lets distributed clients train only the body part of the \ac{ml} model to learn feature representations while keeping the output layer (e.g., the classifier) frozen during the training process inspired by the fact that globally trained model works great for \emph{generalization}, while performs poorly for \textit{personalization}. 
Besides, we perform extensive theoretical analysis to find the theoretical bound of the average global gradient norm of the global loss function.
Our simulation results suggest that while the global trained model performs similarly to \ac{hsfl} in generalization, with the similarly fine-tuned models for both cases, our proposed solution achieves significantly better personalization performance.

\section{System Model and Preliminaries}

\subsection{Network Model}
\noindent
We consider a hierarchical wireless network with $\mathcal{B}=\{b\}_{b=0}^{B-1}$ \acp{es} and a \ac{cs}.
Each \ac{es} has $\mathcal{U}_b=\{u\}_{u=0}^{U_b-1}$ (wireless) clients. 
Besides, we have $\mathcal{U}_b \bigcap \mathcal{U}_{b' \neq b} = \varnothing$ and $\mathcal{U} \coloneqq \bigcup_{b=0}^{B-1} \mathcal{U}_b$.
Denote client $u$'s local dataset by $\mathcal{D}_{u,\mathrm{ft}}=\{\mathbf{x}_n\}_{n=0}^{\mathrm{D}_{u}-1}$ that only contains the feature set.
Furthermore, let us denote the corresponding label set by $\mathcal{D}_{u,\mathrm{lb}} = \{\mathbf{y}_n\}_{n=0}^{\mathrm{D}_u-1}$, which belongs to the \ac{es}.
Moreover, we assume that the clients are resource-constrained and do not own the entire dataset or \ac{ml} model. 
We also assume that the clients have perfect communications with the \acp{es}. 
\bblue{Recall that we split the ML model into a front-end part and a back-end part. 
Therefore, we only need to offload the forward propagation output at the last layer of the front-end part in the uplink and the gradients at the cut layer of the back-end part in the downlink for each \ac{sgd} update. 
Besides, during the model weights aggregation phase, only the weights of the front-end model parts need to be offloaded in the uplink.
As such, the wireless payload sizes in both uplink and downlink are significantly small compared to sharing the entire model payload\footnote{
Since practical networks have cyclic redundancy check, error correction coding and hybrid automatic repeat request process in place \cite[Chap. $13$]{molisch2023wireless}, we assume that these payloads can be successfully offloaded (both in uplink and downlink directions) without any errors.}.}

\subsection{Preliminaries: Hierarchical Federated Learning (HFL)}
\noindent
Similar to the typical single-server-based FL, we want to train a \ac{ml} model $\mathbf{w}$ collaboratively using the clients $\mathcal{U}$ and \acp{es} $\mathcal{B}$ in \ac{hfl}.
In particular, each global round has $\kappa_1$ edge rounds, and each of these edge rounds has $\kappa_0$ local training rounds, i.e., mini-batch \ac{sgd} steps.
Let us denote the indices of the global, edge and local rounds by $t_2$, $t_1$ and $t_0$, respectively.
We thus keep track of the \ac{sgd} steps by 
\begin{align}
    t \coloneqq t_2 \kappa_1 \kappa_0 + t_1 \kappa_0 + t_0.
\end{align}
Denote the global \ac{ml} model at the \ac{cs}, the edge model at the $b^{\mathrm{th}}$ \ac{bs}/\ac{es}\footnote{The terms \ac{bs} and \ac{es} are used interchangeably throughout the rest of the manuscript.} and the local model of user $u \in \mathcal{U}_b$ \bblue{by $Z$-dimensional parameter vectors} $\mathbf{w}^t$, $\mathbf{w}_{b}^t$, and $\mathbf{w}_{u}^t$, respectively.
It is worth noting that the global and edge models are not updated in every \ac{sgd} step in \ac{hfl}, as described below.

At the start of each \emph{global} round $t_2$, i.e., $(t \mod [t_2 \kappa_1 \kappa_0 + t_1 \kappa_0 + t_0]) = 0$, the global model $\mathbf{w}^t$ is broadcasted to all \acp{es}.
These \acp{es} then start their respective \emph{edge} rounds $t_1 \ni t = t_2\kappa_1\kappa_0 + t_1\kappa_0$ by synchronizing their edge models $\mathbf{w}_b^t \gets \mathbf{w}^t$, followed by broadcasting their edge models to their respective clients $\mathcal{U}_b$.
Before the local training, the clients synchronize their local models at $\bar{t}_0 \coloneqq t_2\kappa_2\kappa_1 + t_1 \kappa_0$, as $\mathbf{w}_u^{\bar{t}_0} \gets \mathbf{w}_b^{\bar{t}_0}$.
The clients want to minimize 
\begin{align}
\label{localLossEqn}
    f_u \big(\mathbf{w}_u^{\bar{t}_0} |\mathcal{D}_u \big) 
    \coloneqq \frac{1}{|\mathcal{D}_{u,\mathrm{ft}}|} \sum\nolimits_{{\mathbf{x}_a,y_a} \in \mathcal{D}_u} l(\mathbf{w}_u^{\bar{t}_0} |(\mathbf{x}_a, y_a)) ,
\end{align}
where $\mathcal{D}_u \coloneqq \{\mathcal{D}_{u,\mathrm{ft}}, \mathcal{D}_{u,\mathrm{lb}} \}$ and $l(\mathbf{w}_u^{\bar{t}_0}|(\mathbf{x}_a, y_a)$ the loss function (e.g., cross entropy, mean square error, etc.) evaluated using \ac{ml} model $\mathbf{w}_u^{\bar{t}_0}$ and training sample $(\mathbf{x}_a, y_a)$.
Each user takes $\kappa_0$ \ac{sgd} steps to minimize (\ref{localLossEqn}) as 
\begin{align}
    \mathbf{w}_{u}^{\bar{t}_0 + \kappa_0} = \mathbf{w}_{u}^{\bar{t}_0} - \eta \sum\nolimits_{t_0=0}^{\kappa_0-1} g_u (\mathbf{w}_{u}^{t_0}) ,
\end{align}
where $\eta$ is the step size and $g_u(\cdot)$ is the stochastic gradient
Once $\kappa_0$ \ac{sgd} steps get completed, each client offloads the trained model $\mathbf{w}_{u,b}^{\bar{t}_0+\kappa_0}$ to their respective associated \ac{es}, and the \ac{es} then aggregates the updated models as 
\begin{align}
    \mathbf{w}_{b}^{\bar{t}_0 + \kappa_0} = \sum\nolimits_{u \in \mathcal{U}_b} \alpha_u 
    \cdot \mathbf{w}_u^{{\bar{t}_0}+\kappa_0},  
\end{align}
where $0\leq \alpha_u \leq 1$ and $\sum_{u\in\mathcal{U}_b} \alpha_u=1$ for all $b \in \mathcal{B}$.
This completes one \emph{edge} round and hence, \ac{es} minimizes the following loss function 
\begin{align}
\label{esLossFunc}
    f_b (\mathbf{w}) = \sum\nolimits_{u \in\mathcal{U}_b} \alpha_u f_u (\mathbf{w}).
\end{align}
Each \ac{es} repeats the above steps for $\kappa_1$ times and then sends the updated edge model to the \ac{cs} at $t=t_2\kappa_1\kappa_0 + \kappa_1\kappa_0$.
The \ac{cs} then updates the global model as
\begin{align}
    \mathbf{w}^{(t_2+1) \kappa_1\kappa_0} = \sum\nolimits_{b=0}^{B-1} \alpha_b 
    \cdot \mathbf{w}_b^{t_2\kappa_1\kappa_0 + \kappa_1\kappa_0},  
\end{align}
where $0\leq \alpha_b \leq 1$ and $\sum_{b=0}^{B-1} \alpha_b = 1$.
This completes one global round, and hence, the \ac{cs} minimizes the following global loss function
\begin{align}
\label{globalLossFunc}
    f (\mathbf{w}) = \sum\nolimits_{b=0}^{B-1} \alpha_b f_b (\mathbf{w}) = \sum\nolimits_{b=0}^{B-1} \alpha_b \sum\nolimits_{u \in\mathcal{U}_b} \alpha_u f_u (\mathbf{w}).
\end{align}

\section{Personalized Split Hierarchical FL}
\noindent
To mitigate the resource constraints and achieve good \emph{personalization} ability, we propose a \ac{phsfl} algorithm that first trains a global model with SL and a frozen output/classifier layer and then lets the clients fine-tune only the classifier on their local training dataset.

\subsection{Proposed Personalized Split HFL: Global Model Training}
\noindent
First, we describe the global model training process with the following key steps.

\subsubsection{\textbf{Step 1 - Global Round Initialization}}
At each global round $t_2$, the global model $\mathbf{w}^{t_2\kappa_1\kappa_0}$ is broadcasted to all \acp{es}.
\subsubsection{\textbf{Step 2 - Edge Round Initialization}}
At the start of the first \emph{edge} round, i.e., $t_1=0$, each \ac{es} \emph{synchronizes} its edge model as 
\begin{align}
    \mathbf{w}_b^{t_2\kappa_1\kappa_0 + t_1\kappa_0} \gets \mathbf{w}^{t_2\kappa_1\kappa_0}.    
\end{align}
\noindent
\textbf{Step 2.1 - Edge Model Splitting}: Each \ac{es} then splits the model at the cut layer as $\mathbf{w}_{b}^{\bar{t}_0} \coloneqq [\mathbf{w}_{b,0}^{\bar{t}_0}; \mathbf{w}_{b,1}^{\bar{t}_0}]$ \cite{vepakomma2018split}, where recall that $\bar{t}_0 \coloneqq t_2\kappa_1\kappa_0 + t_1\kappa_0$, \bblue{$\mathbf{w}_{b,0}^{\bar{t}_0} \in \mathbb{R}^{Z_0}$ and $ \mathbf{w}_{b,1}^{\bar{t}_0} \in \mathbb{R}^{Z-Z_0}$}.
While \ac{sl} helps us to tackle the resource constraints, it does not ensure model \emph{personalization}. 
As such, we first split the server-side model into two parts as $\mathbf{w}_{b,1}^{\bar{t}_0} = [\mathbf{w}_{b,1,\mathrm{bd}}^{\bar{t}_0}; \mathbf{w}_{b,1,\mathrm{hd}}^{\bar{t}_0}]$, where $\mathbf{w}_{b,1,\mathrm{bd}}^{\bar{t}_0}$ is the \emph{body} part that works as the feature extractor, $\mathbf{w}_{b,1,\mathrm{hd}}^{\bar{t}_0}$ is the \emph{head} part, i.e., the output layer/classifier that generates the final output.

\noindent
\textbf{Step 2.2 - Freeze Server-Side Classifier}:
Following \cite{oh2022fedbabu}, we randomly initiate $\mathbf{w}_{b,1,\mathrm{hd}}^{\bar{t}_0}$ and freeze it during the training phase, i.e., $\mathbf{w}_{b,1,\mathrm{hd}}^{\bar{t}_0}$ is never updated during the model training phase.

\noindent
\textbf{Step 2.3 - Edge Model Broadcasting}:
Each \ac{es} then only broadcasts the client-side model parts $\mathbf{w}_{b,0}^{\bar{t_0}}$ to its associated users $\mathcal{U}_b$.

\subsubsection{\textbf{Step 3 - Local Model Training}}
Since clients can only train the client-side model, local training requires information exchange between the client-side and server-side models. 
This is achieved through the following key steps.

\noindent
\textbf{Step 3.1 - Local Model Synchronization}:
Each clients in $\mathcal{U}_b$ synchronizes their respective local model, i.e., the \emph{client-side} model part, using the received broadcasted model as $\mathbf{w}_{u,0}^{\bar{t}_0} \gets \mathbf{w}_{b,0}^{\bar{t}_0}$.
The \ac{es} initializes the \emph{server-side} model parts for each $u \in \mathcal{U}_b$ as $\mathbf{w}_{u,1}^{\bar{t}_0} \gets \mathbf{w}_{b,1}^{\bar{t}_0}$ simultaneously.

\noindent
\textbf{Step 3.2 - Forward Propagation of Client-Side Model}: 
Given input data $\mathbf{x}_n$, the output of the activation function is a mapping of the input data and corresponding weight matrices of the layers in the client-side model. 
More specifically, we consider mini-batch \ac{sgd}, where each client randomly selects $N$, i.e., mini-batch size, samples, and feeds that as the input to the model. 
Denote the indices of the randomly selected training features by $\mathcal{N}$.
As such, denote the output at the cut-layer during the $t_0^{\mathrm{th}}$ local round for a mini-batch of $N$ training samples as $\mathbf{o}_{u, \mathrm{fp}}^{t_0} \coloneqq \tilde{f} \big( \mathbf{w}_{u,0}^{t_0} | \{\mathbf{x}_n\}_{n=0}^{N} \big)$, where $\mathbf{x}_n \sim \mathcal{D}_{u,\mathrm{ft}}$ and $\tilde{f} (\mathbf{w}_{u,0}^{\bar{t}_0} | \{\mathbf{x}_n\}_{n=0}^{N})$ represents the \emph{forward propagation} with respect to model $\mathbf{w}_{u,0}^{\bar{t}_0}$ conditioned on the data samples $\{\mathbf{x}_n\}_{n=0}^N$. 
\bblue{Besides, $\mathbf{o}_{u, \mathrm{fp}}^{t_0} \in \mathbb{R}^{N \times Z_{\mathrm{c}}}$, where $Z_{\mathrm{c}} \ll Z_0$ is the cut-layer size.}

\noindent 
\textbf{Step 3.4 - Transmission of Cut-Layer's Output}: 
Each client offloads $\mathbf{o}_{u, \mathrm{fp}}^{\bar{t}_0}$ and the set of (randomly sampled) indices $\mathcal{N}$ to its associated \ac{es}.
It is worth noting that existing \ac{hsfl} algorithms usually require offloading the corresponding labels to the associated server(s), which may reveal private information. 

\noindent
\textbf{Step 3.5 - Forward Propagation of Server-Side Model}: 
Each \ac{es} takes $\mathbf{o}_{u, \mathrm{fp}}^{\bar{t}_0}$ as input and computes the \emph{forward propagation} of the \emph{server-side} model part that gives the predicted label $\hat{\mathbf{Y}} \coloneqq \tilde{f} \big( \mathbf{w}_{u,1}^{\bar{t}_0} | \mathbf{o}_{u, \mathrm{fp}}^{\bar{t}_0}, \mathcal{N} \big)$.
The \ac{es} extracts the original labels using the received indices $\mathcal{N}$.
Given the original label for the sample $\mathbf{x}_n$ is $\mathbf{y}_n$, the loss associated to this particular sample is denoted by $l \big(\mathbf{y}_n, \hat{\mathbf{Y}} [n]\big) = l \big( \mathbf{y}_n, \tilde{f} \big( \mathbf{w}_{u,1}^{\bar{t}_0} | \mathbf{o}_{u, \mathrm{fp}}^{\bar{t}_0},\mathcal{N} \big)[n] \big) $.
As such, we write the loss function associated with a mini-batch as 
\begin{align}
\label{localLossServerSide}
    &f_{u} \big(\mathbf{w}_{u, 1}^{\bar{t}_0} | \mathbf{o}_{u, \mathrm{fp}}^{\bar{t}_0} \big) 
    \coloneqq [1/N] \sum\nolimits_{n \in \mathcal{N}}  l(\mathbf{y}_n, \hat{\mathbf{Y}}[n]) \nonumber \\
    &\qquad \qquad= [1/N] l \sum\nolimits_{n \in \mathcal{N}} \big( \mathbf{y}_n, \tilde{f} \big( \mathbf{w}_{u,1}^{\bar{t}_0} | \mathbf{o}_{u, \mathrm{fp}}^{\bar{t}_0},\mathcal{N} \big)[n] \big).
\end{align}

\noindent
\textbf{Step 3.6 - Back Propagation of Server-Side Model}:
In order to minimize (\ref{localLossServerSide}), the server first performs \emph{backpropagation} using the server-side model part as 
\begin{align}
\label{bpropServerSide}
    & \mathbf{w}_{u,1}^{\bar{t}_0+1} = 
    \mathbf{w}_{u,1}^{\bar{t}_0} - \eta \big[g_{u} \big(\mathbf{w}_{u,1}^{\bar{t}_0} | \mathbf{o}_{u, \mathrm{fp}}^{\bar{t}_0} \big)\big]_{\mathbf{w}_{u, 1}^{\bar{t}_0}},
\end{align}
where $\eta$ is the learning rate and $\big[g_{u} \big(\mathbf{w}_{u,1}^{\bar{t}_0} | \mathbf{o}_{u, \mathrm{fp}}^{\bar{t}_0} \big)\big]_{\mathbf{w}_{u, 1}^{\bar{t}_0}}$, represents the stochastic gradient with respect to the server-side model parameters $\mathbf{w}_{u,1}^{\bar{t}_0}$. 
Since $\mathbf{w}_{u,1,\mathrm{hd}}^{\bar{t}_0}$ is frozen and $\mathbf{w}_{u,1}^{\bar{t}_0} = [\mathbf{w}_{u,1,\mathrm{bd}}^{\bar{t}_0}; \mathbf{w}_{u,1,\mathrm{hd}}^{\bar{t}_0}]$, (\ref{bpropServerSide}) implies the following   
\begin{align}
    & \mathbf{w}_{u,1,\mathrm{bd}}^{\bar{t}_0+1} = 
    \mathbf{w}_{u,1,\mathrm{bd}}^{\bar{t}_0} - \eta \big[g_{u} \big(\mathbf{w}_{u,1,\mathrm{bd}}^{\bar{t}_0} | \mathbf{o}_{u, \mathrm{fp}}^{\bar{t}_0} \big)\big]_{\mathbf{w}_{u,1,\mathrm{bd}}^{\bar{t}_0}}.\\
    &\mathbf{w}_{u,1,\mathrm{hd}}^{\bar{t}_0+1} = 
    \mathbf{w}_{u,1,\mathrm{hd}}^{\bar{t}_0} - 0 \times \big[g_{u} \big(\mathbf{w}_{u,1,\mathrm{hd}}^{\bar{t}_0} | \mathbf{o}_{u, \mathrm{fp}}^{\bar{t}_0} \big)\big]_{\mathbf{w}_{u,1,\mathrm{hd}}^{\bar{t}_0}}, \label{headUpdate}
\end{align}
where a learning rate of $0$ means the $\mathbf{w}_{u,1\mathrm{hd}}^{\bar{t}_0}$ is not getting updated.
Besides, we do not need to set $\eta=0$ explicitly for the classifier during the model training phase since we can disable the gradient computation for the classifier in almost all popular \ac{ml} libraries like PyTorch\footnote{\url{https://pytorch.org/}} and TensorFlow\footnote{\url{https://www.tensorflow.org/}}.
Furthermore, the gradients are calculated using the chain rule. 
Denote the gradient of the \emph{server-side} input layer, i.e., the layer that receives the cut-layer output $\mathbf{o}_{u, \mathrm{fp}}^{\bar{t}_0} \bblue{\in \mathbb{R}^{N\times Z_{\mathrm{c}}}}$ as \emph{client-side} model's input by $ \mathbf{o}_{u, \mathrm{bp}}^{\bar{t}_0}$.

\noindent
\textbf{Step 3.7 - Transmission of Server-Side Cut-Layer's Gradient}: 
The \ac{es} then transmits the gradient $ \mathbf{o}_{u, \mathrm{bp}}^{\bar{t}_0}$ to the client to compute the gradients of the client-side model $\mathbf{w}_{u,0}^{\bar{t}_0}$ using the chain rule.

\noindent
\textbf{Step 3.8 - Back Propagation of Client-Side Model}: 
Each client then performs \emph{backpropagation} to compute the gradients that minimizes the loss function as 
\begin{align}
\label{bpropClientSide}
    &\mathbf{w}_{u,0}^{\bar{t}_0+1} = \mathbf{w}_{u,0}^{\bar{t}_0} - \eta \big[ g_{u} \big( \mathbf{w}_{u,1}^{\bar{t}_0} | \mathbf{o}_{u, \mathrm{fp}}^{\bar{t}_0} \big) \big]_{\mathbf{w}_{u,0}^{\bar{t}_0} | \mathbf{o}_{u, \mathrm{bp}}^{\bar{t}_0} },
\end{align}
where the notation $\big[ g_{u} \big( \mathbf{w}_{u,1}^{\bar{t}_0} | \mathbf{o}_{u, \mathrm{fp}}^{\bar{t}_0} \big) \big]_{\mathbf{w}_{u,0}^{\bar{t}_0} | \mathbf{o}_{u, \mathrm{bp}}^{\bar{t}_0} }$ represents the stochastic gradient of the loss function with respect to $\mathbf{w}_{u,0}^{\bar{t}_0}$ that depends on the gradient at cut-layer $\mathbf{o}_{u, \mathrm{bp}}^{\bar{t}_0}$, which is available to the client. 
Therefore, the client can complete the gradient with respect to its \emph{client-side} model parameters using chain rule.

\textbf{Steps 3.2 - 3.8} are repeated for $\mathrm{N}$ mini-batches, which completes one local training epoch. 
Each client performs $\kappa_0$ local epochs, which complete the local training.

\subsubsection{\textbf{Step 4 - Trained Client-Side Model Offloading}}
All clients offload their respective trained \emph{client-side} model $\mathbf{w}_{u,0}^{\bar{t}_0 + \kappa_0}$ to their associated \ac{es}. 

\subsubsection{\textbf{Step 5 - Edge Aggregation}} 
Each \ac{es} then aggregates the client-side and server-side models as 
\begin{align}
    &\mathbf{w}_{b,0}^{\bar{t}_0 + \kappa_0} = \mathbf{w}_{b,0}^{t_2\kappa_1 \kappa_0 + (t_1+1)\kappa_0}  
    = \sum\nolimits_{u=0}^{U_b-1} \alpha_u \mathbf{w}_{u,0}^{\bar{t}_0 + \kappa_0}.\\
    &\mathbf{w}_{b,1}^{\bar{t}_0 + \kappa_0} = \mathbf{w}_{b,1}^{t_2\kappa_1\kappa_0 + (t_1+1)\kappa_0} 
    = \sum\nolimits_{u=0}^{U_b-1} \alpha_u \mathbf{w}_{u,1}^{\bar{t}_0+\kappa_0}.
\end{align}

Each \ac{es} repeats \textbf{Step 2.2} to \textbf{Step 5} $\kappa_1$ times and then sends the updated edge model $\mathbf{w}_{b}^{t_2\kappa_1\kappa_2 + \kappa_1\kappa_2} = [\mathbf{w}_{b,0}^{t_2\kappa_1\kappa_2 + \kappa_1\kappa_2}; \mathbf{w}_{b,1}^{t_2\kappa_1\kappa_2 + \kappa_1\kappa_2}]$ to the \ac{cs}.

\subsubsection{\textbf{Step 6 - Global Aggregation}} 
Upon receiving the updated edge models, the \ac{cs} aggregates these updated models and updates the global model as
\begin{align}
    \mathbf{w}^{(t_2+1)\kappa_1\kappa_0} = \sum\nolimits_{b=0}^{B-1} \alpha_b \mathbf{w}_{b}^{t_2\kappa_1\kappa_0 + \kappa_1\kappa_0}.
\end{align}
This concludes one global round.
The above steps are repeated for $t_2=0,1,\dots,R-1$ rounds.

\begin{Remark}[Communication overheads]
The communication overhead to offload $\mathbf{o}_{u, \mathrm{fp}}^{t_0}$ is $ \left( N \times Z_{\mathrm{c}} \right) \times (\omega + 1)$, where $\omega$ is the floating point precision \cite{pervej2023resource}.
Besides, the overhead to offload $N$ indices is upper bounded by $N \times \left(\left\lceil \log_2\left(|\mathcal{D}_{u,\mathrm{ft}}|\right) \right\rceil + 1\right)$, where $\left\lceil \cdot \right\rceil$ is the \emph{ceiling} function. 
Therefore, for each local round, i.e., $\mathrm{N}$ mini-batch \ac{sgd} steps, the cumulative communication overhead is upper bounded by $\Phi_{\mathrm{local}} \coloneqq \mathrm{N} \times \left\{2 \left[\left( N \times Z_{\mathrm{c}} \right) (\omega + 1)\right] + N \left(\left\lceil \log_2\left(|\mathcal{D}_{u,\mathrm{ft}}|\right) \right\rceil + 1\right) \right\}$.
Furthermore, we have a communication overhead of $\Phi_{\mathrm{off}} \coloneqq Z_0 \times (\omega+1)$ bits, where $Z_0$ is the total number of parameters in $\mathbf{w}_{b,0}^t$, during the trained \emph{client-side} model offloading phase.
As such, for a single edge aggregation round, our proposed \ac{phsfl} has the following cumulative communication overhead
\begin{align}
    \Phi_{\mathrm{PHSFL}} \leq  \kappa_0 \cdot \Phi_{\mathrm{local}} + 2\Phi_{\mathrm{off}}, 
\end{align}
which is due to the fact that each client has to perform $\kappa_0$ local rounds and the \emph{client-side} model $\mathbf{w}_{b,0}^t$ has to be broadcasted/offloaded in the downlink/uplink. 
Note that the communication overhead with typical \ac{hfl} is $\Phi_{\mathrm{HFL}} \coloneqq 2 \times Z (\omega+1)$. 
Therefore, \ac{phsfl} is communication efficient only when $\Phi_{\mathrm{HFL}} > \Phi_{\mathrm{PHSFL}}$, which is typically the case since $Z \gg (Z_0 + Z_\mathrm{c})$.
\end{Remark}

\begin{Remark}[Choice of the cut layer]
As practical networks are resource constraints, the choice of the cut layer should depend on the clients' and \ac{es}'s available resources. 
While it is also possible to optimize the cut layer\footnote{\bblue{However, it can become vital to optimize the cut layer under extreme delay and energy constraints. 
In such cases, one may jointly consider the learning and wireless networking constraints to optimize the cut layer and other parameters, which is beyond the scope of this paper.}}, the choice of this cut layer does not affect the training performance as long as the clients can train the client-side model block and offload it back to the \ac{es} during the model aggregation phase without any errors.
This is due to the fact that we calculate the loss based on the input $\mathbf{x}_n \sim \mathcal{D}_{u,\mathrm{ft}}$, which we minimize by calculating the gradients. 
\end{Remark}

\subsection{Personalized Split HFL: Fine-Tuning of Global Trained Model}
\label{personalizationSteps}
\noindent
Once the globally trained model $\mathbf{w}^{*}$ is obtained, the \ac{cs} broadcasts it to all \acp{es}.
We assume the clients will \emph{fine-tune} the trained model for $\{k\}_{k=0}^{K-1}$ \ac{sgd} steps.
The \ac{es} then splits the model as $\mathbf{w}_{b}^{k} \coloneqq [\mathbf{w}_{b,0}^{k}; \mathbf{w}_{b,1}^{k}]$.
It then broadcasts the trained \emph{client-side} model to all clients and initialize the \emph{server-side} model as $\mathbf{w}_{u,1}^{k} = [\mathbf{w}_{u,1,\mathrm{bd}}^{k}; \mathbf{w}_{u,1,\mathrm{hd}}^{k}]$.
The client computes its forward propagation output $\mathbf{o}_{u, \mathrm{fp}}^{k}$ and shares that with the \ac{es}. 
The \ac{es} then takes a \ac{sgd} step to fine-tune the head/classifier as 
\begin{align}
    \mathbf{w}_{u,1,\mathrm{hd}}^{k+1} = \mathbf{w}_{u,1,\mathrm{hd}}^{k} - \tilde{\eta} \big[g_{u} \big(\mathbf{w}_{u,1,\mathrm{hd}}^{k} | \mathbf{o}_{u, \mathrm{fp}}^{k} \big)\big]_{\mathbf{w}_{u,1,\mathrm{hd}}^{k}},
\end{align}
where $\tilde{\eta}$ is the learning rate that can differ from $\eta$.
The body part and the client-side model weights remain \emph{as-it-is} in $\mathbf{w}^{*}$.
The above steps are repeated for $K$ times, which gives the personalized classifier as $\mathbf{w}_{u,1,\mathrm{hd}}^{K}$.
Hence, the entire personalized model of the $u^{\mathrm{th}}$ user can be expressed as $\mathbf{w}_u^{K} = [\mathbf{w}_{b,0}^{*}; [\mathbf{w}_{b,1,\mathrm{bd}}^{*}; \mathbf{w}_{b,1,\mathrm{hd}}^{K}]]$.

\section{Theoretical Analysis}
\noindent
For notational simplicity, we use the notation $f_u(\mathbf{w}_u^t)$ to represent client's loss function $f_{u} \big(\mathbf{w}_{u, 1}^{t} | \mathbf{o}_{u, \mathrm{fp}}^{t} \big)$ and $g_u (\mathbf{w}_u^t)$ to represent $\left[\big[ g_{u} \big( \mathbf{w}_{u,1}^{t} | \mathbf{o}_{u, \mathrm{fp}}^{t} \big) \big]_{\mathbf{w}_{u,0}^{t} | \mathbf{o}_{u, \mathrm{bp}}^{t}}; \big[g_{u} \big(\mathbf{w}_{u,1}^{t} | \mathbf{o}_{u, \mathrm{fp}}^{t} \big)\big]_{\mathbf{w}_{u, 1}^{t}} \right]$ throughout the rest of the paper.

\subsection{Assumptions}
\noindent
We make the following standard assumptions \cite{wang2020tackling, pervej2024hierarchical,wang2022demystifying, pervej2025resource, ye2023feddisco} 
\begin{Assumption}[Smoothness] 
    The loss functions are $\beta$-Lipschitz smooth, i.e., for some $\beta>0$, $\Vert \nabla f_u(\mathbf{w}) - \nabla f_u(\mathbf{w}') \Vert \leq \beta \Vert \mathbf{w} - \mathbf{w}' \Vert$, where $\Vert \cdot \Vert$ is the $L_2$ norm, for all loss functions.
\end{Assumption}

\begin{Assumption}[Unbiased gradient with bounded variance]
    The mini-batch stochastic gradient calculated using client's randomly sampled mini-batch $\zeta$ is unbiased, i.e., $\mathbb{E}_{\zeta \sim \mathcal{D}_u} \left[g_u(\mathbf{w}) \right] = \nabla f_u (\mathbf{w})$, where $\mathbb{E} [\cdot]$ is the expectation operator.
    Besides, the variance of the gradients is bounded, i.e., $ \mathbb{E}_{\zeta \sim \mathcal{D}_u} \left[ \Vert \nabla f_u (\mathbf{w}) - g_u(\mathbf{w}) \Vert^2 \right] \leq \sigma^2 $, for some $\sigma \geq 0$ and all $u \in \mathcal{U}$.
\end{Assumption}

\begin{Assumption}[Bounded gradient divergence]
    The divergence between (a) local and \ac{es} loss functions and (b) \ac{es} and global loss functions are bounded as
    \begin{align}
        \sum\nolimits_{u \in \mathcal{U}_b} \alpha_u \Vert \nabla f_u (\mathbf{w}) - \nabla f_b (\mathbf{w}) \Vert^2 \leq \epsilon_0^2, \\
        \sum\nolimits_{b=1}^B \alpha_b \Vert \nabla f_b (\mathbf{w}) - \nabla f (\mathbf{w}) \Vert^2 \leq \epsilon_1^2, 
    \end{align}
    for some $\epsilon_0 \leq 0$, $\epsilon_1 \leq 0$ and all $u$ and $b$.
\end{Assumption}

Additionally, since the global and edge models do not exist in every \ac{sgd} step $t$, following standard practice \cite{pervej2024hierarchical,wang2022demystifying}, we assume that virtual copies of these models exist, denoted by $\bar{\mathbf{w}}^t$ and $\bar{\mathbf{w}}_b^t$, respectively.
Moreover, the above assumptions also apply to these virtual models.

\subsection{Convergence Analysis}
\begin{Theorem}
Suppose the above assumptions hold. Then, if the learning rate satisfies $\eta < \frac{1}{2\sqrt{5} \beta \kappa_1 \kappa_0}$, the average global gradient norm is upper bounded by
\begin{align}
\label{covBound_Eqn}
    &\rs\rs \frac{1}{T} \sum_{t=0}^{T-1} \mathbb{E} \left[\left\Vert \nabla f (\bar{\mathbf{w}}^t) \right\Vert^2 \right] 
    \leq \frac{2 \left(\mathbb{E} \left[f \left( \bar{\mathbf{w}}^0 \right) \right] - \mathbb{E} \left[f \left( \bar{\mathbf{w}}^{T} \right) \right]\right)}{\eta T} + \nonumber\\
    &\rs\rs \beta \eta \sigma^2 \sum\nolimits_{b=0}^{B-1} \alpha_b^2 \sum\nolimits_{u \in \mathcal{U}_b} \alpha_u^2 + \left(\Gamma_0 + \Gamma_1\right)\sigma^2 + \tilde{\Gamma}_0 \epsilon_0^2 +  \tilde{\Gamma_1}\epsilon_1^2, \rs
\end{align}
where $\Gamma_0 \coloneqq  4\beta^2 \eta^2 \kappa_0^2 - 4 \beta^2 \eta^2 \kappa_0^2  \sum_{b=0}^{B-1} \alpha_b \sum_{u \in \mathcal{U}_b} \alpha_{u}^2$, 
$\Gamma_1 \coloneqq 80 \kappa_1^2 \beta^4 \eta^4 \kappa_0^4 + 4 \kappa_1 \kappa_0 \beta^2 \eta^2 \sum_{b=0}^{B-1} \alpha_b \sum_{u \in \mathcal{U}_b} \alpha_{u}^2 - 4 \kappa_1 \kappa_0 \beta^2 \eta^2 \sum_{b=0}^{B-1} \alpha_{b}^2 \sum_{u\in \mathcal{U}_{b}} \alpha_u^2 - 80 \kappa_1^2 \beta^4 \eta^4 \kappa_0^4 \sum_{b=0}^{B-1} \alpha_b \sum_{u \in \mathcal{U}_b} \alpha_{u}^2$, $\tilde{\Gamma}_0 \coloneqq 12 \beta^2 \eta^2 \kappa_0^2 \big(1 + 20 \kappa_0^2\kappa_1^2 \beta^2 \eta^2 \big)$ and $\tilde{\Gamma}_1 \coloneqq 20 \beta^2 \eta^2 \kappa_1^2 \kappa_0^2$.
\end{Theorem}

\begin{proof}
The proof is left in the appendix. 
\end{proof}

\begin{Remark}
The first two terms in (\ref{covBound_Eqn}) are analogous to standard \ac{sgd}: the first term captures changes in the loss functions, while the second term appears from the bounded variance assumption of the stochastic gradients. 
The third term also appears due to the bounded variance assumption of the stochastic gradients: $\Gamma_0$ and $\Gamma_1$ are the contribution from user-\ac{es} and \ac{es}-\ac{cs} hierarchy levels, respectively.
The fourth and fifth terms capture the divergence between the user-\ac{es} and \ac{es}-\ac{cs} loss functions, respectively, and arise from the statistical data heterogeneity.
\end{Remark}

\begin{figure}[!t]
\begin{subfigure}{0.242\textwidth}
    \centering
    \includegraphics[trim=5 6 40 20, clip, width=\linewidth]{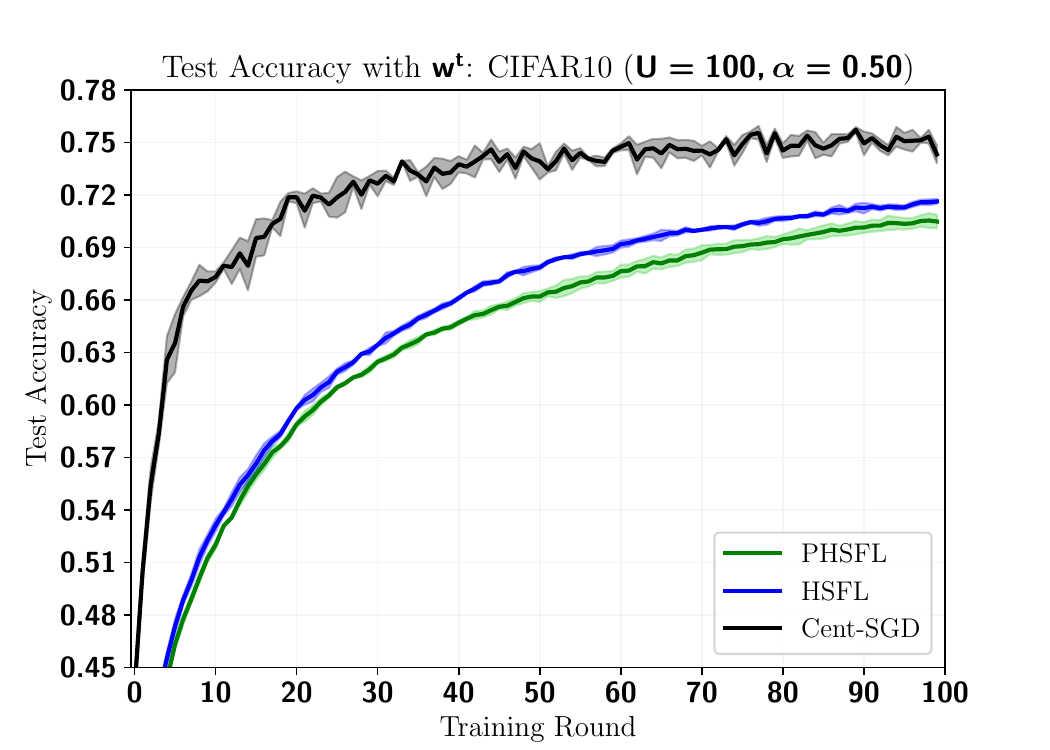}
    \caption{Test accuracy with $\mathbf{w}^t$}
    \label{trainVstestAcc_alpha_0.5}
\end{subfigure}
\begin{subfigure}{0.24\textwidth}
    \centering
    \includegraphics[trim=5 6 40 20, clip, width=\linewidth]{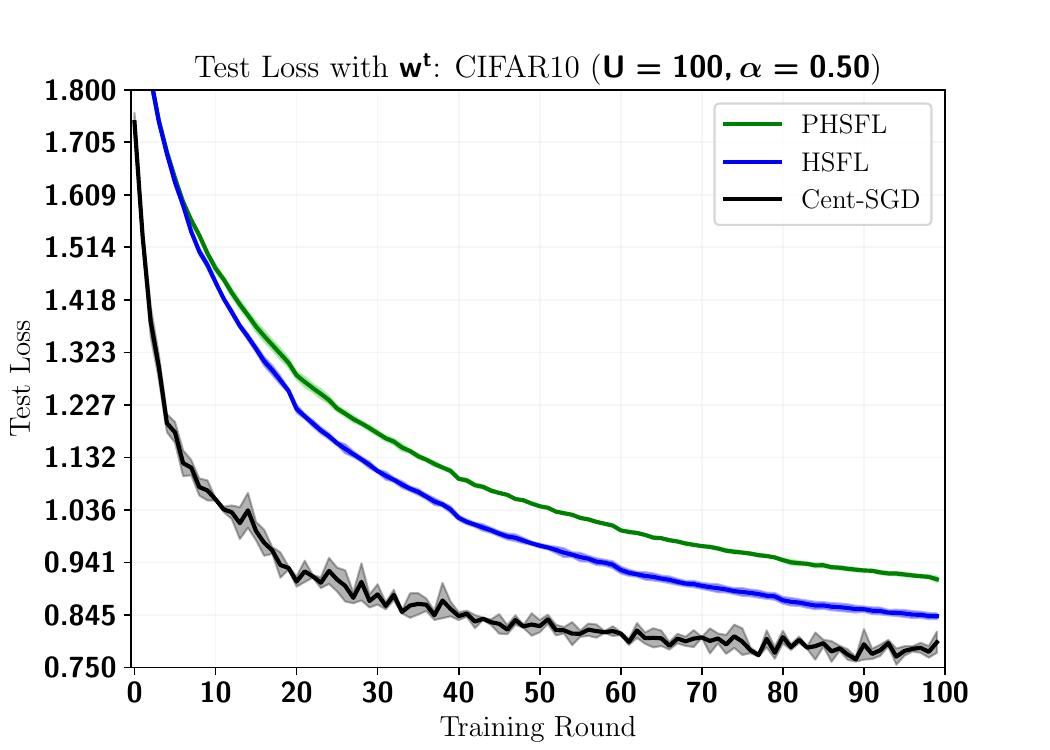}
    \caption{Test loss with $\mathbf{w}^t$}
    \label{trainVstestLoss_alpha_0.5}
\end{subfigure}
\caption{Test performance comparisons on CIFAR$10$ across $U=100$ users, when $\mathrm{Dir}(\pmb{\alpha}=\mathbf{0.5)}$}
\label{gRoundVstestPerformance_alpha_0.5}
\end{figure}

\section{Simulation Results and Discussions}

\subsection{Simulation Setting}
\noindent
We consider $4$ \acp{bs}/\acp{es}, each serving $25$ users.
For simplicity, we assume perfect communication between the users and the \acp{bs} for information exchange.
Besides, we use a symmetric Dirichlet distribution ${\tt Dir(\pmb{\alpha})}$, where $\pmb{\alpha}$ is the concentration parameter and determines the skewness, to distribute the training and test samples across these $U=100$ users following a similar strategy as in \cite{pervej2023resource}. 
We use the image classification task with the widely popular CIFAR$10$ dataset to evaluate the performance of our proposed \ac{phsfl} algorithm.
However, this can be easily extended to other applications and/or datasets.

We use a simple \ac{cnn} model as our \ac{ml} model $\mathbf{w}$ that has the following architecture: {\tt Conv2d (\#Channels, $64$) $\rightarrow$ ReLU() $\rightarrow$ MaxPool2d $\rightarrow$ Conv2d($64$,$128$) $\rightarrow$ ReLU() $\rightarrow$ MaxPool2d $\rightarrow$ FC($512$, $256$), $\rightarrow$ ReLU() $\rightarrow$ FC($256$, \#Labels)}. 
The model is split after the first {\tt MaxPool2d} layer.
As such, the \emph{client-side} model part is {\tt \tt Conv2d (\#Channels, $64$) $\rightarrow$ ReLU() $\rightarrow$ MaxPool2d}.
The rest of the model blocks belong to the \emph{server-side} model parts, where the \emph{classifier/head} is the output layer, i.e., {\tt FC($256$, \#Labels)}.
Finally, for the model training, we use {\tt \ac{sgd}} optimizer with $\eta=0.01$, $\mathrm{N}=32$, $N=5$, $\kappa_0=5$, $\kappa_1=3$ and $R=100$.

\begin{figure}[!t]
\begin{subfigure}{0.242\textwidth}
    \centering
    \includegraphics[trim=5 0 50 20, clip, width=\linewidth]{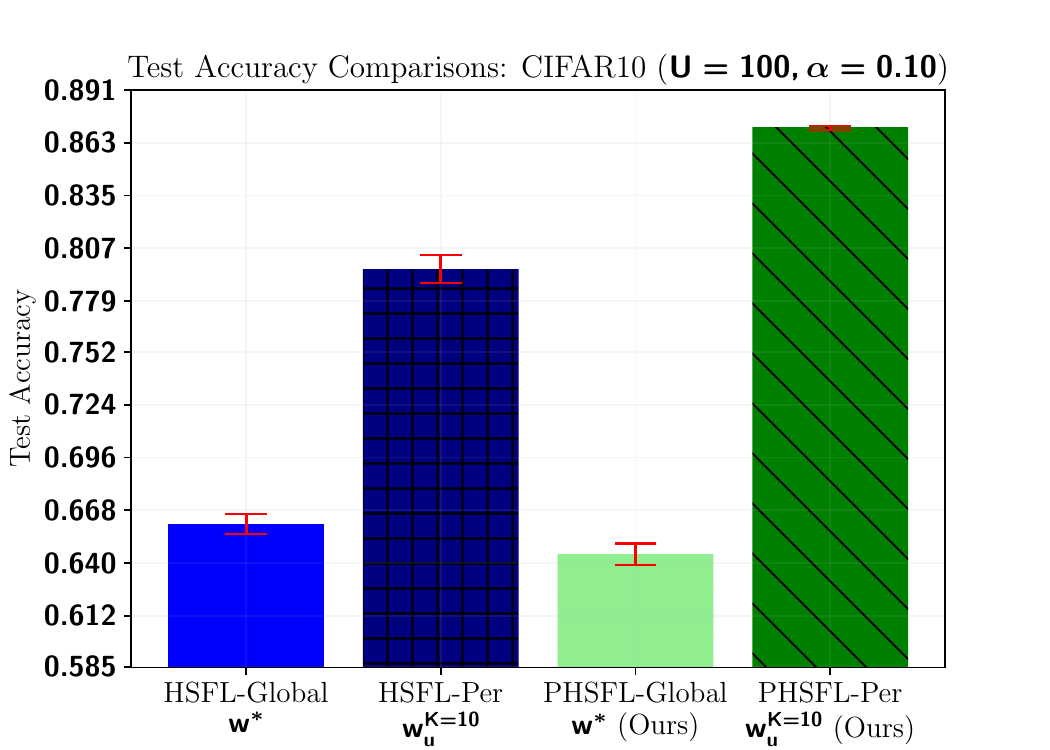}
    \caption{Test Accuracy Comparisons}
    \label{testAcc_alpha_0.1}
\end{subfigure}
\begin{subfigure}{0.24\textwidth}
    \centering
    \includegraphics[trim=5 0 50 20, clip, width=\linewidth]{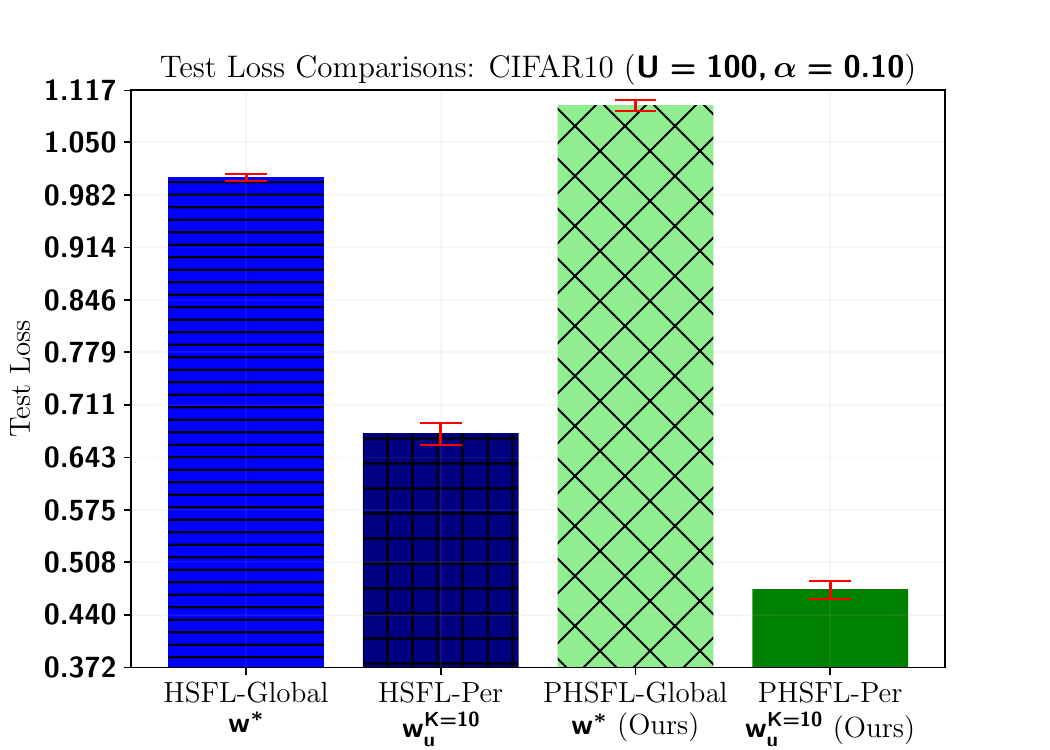}
    \caption{Test Loss Comparisons}
    \label{testLoss_alpha_0.1}
\end{subfigure}
\caption{Test performance comparisons on CIFAR$10$ across $U=100$ users, when $\mathrm{Dir}(\pmb{\alpha}=\mathbf{0.1)}$\bblue{: $\mathbf{w}^{*}$ and $\mathbf{w}_u^{K=10}$ represent the global trained model and fine-tuned personalized model, respectively}}
\label{testPerformance_alpha_0.1}
\end{figure}
\begin{figure}[!t]
\begin{subfigure}{0.242\textwidth}
    \centering
    \includegraphics[trim=5 0 50 20, clip, width=\linewidth]{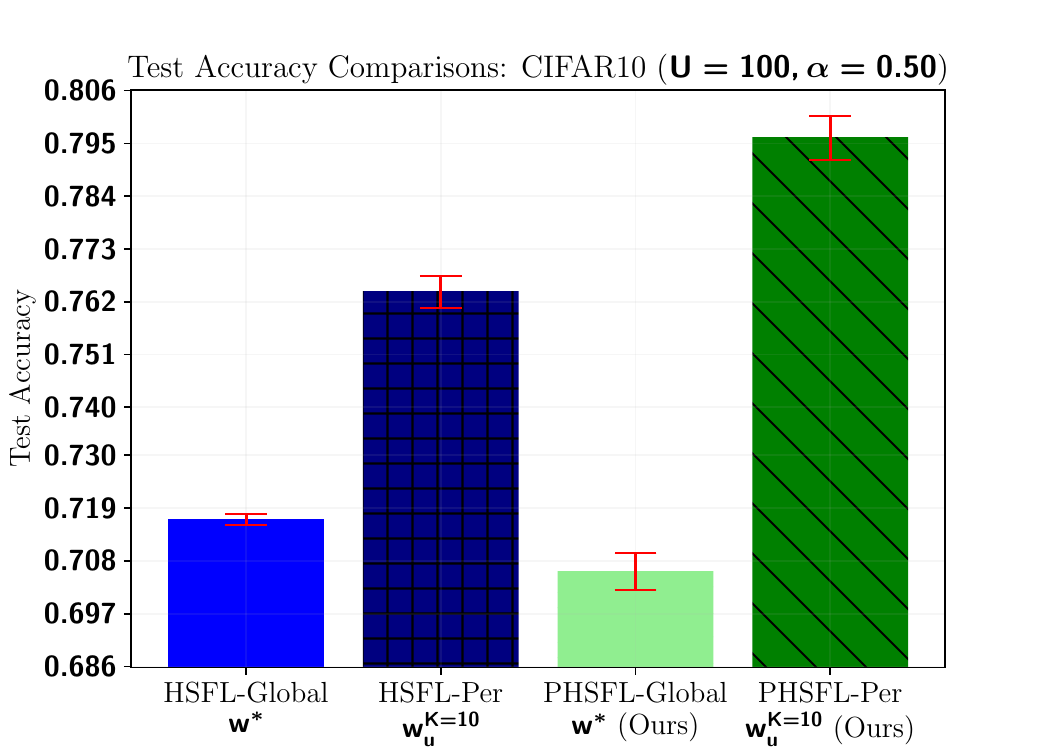}
    \caption{Test Accuracy Comparisons}
    \label{testAcc_alpha_0.5}
\end{subfigure}
\begin{subfigure}{0.24\textwidth}
    \centering
    \includegraphics[trim=5 0 50 20, clip, width=\linewidth]{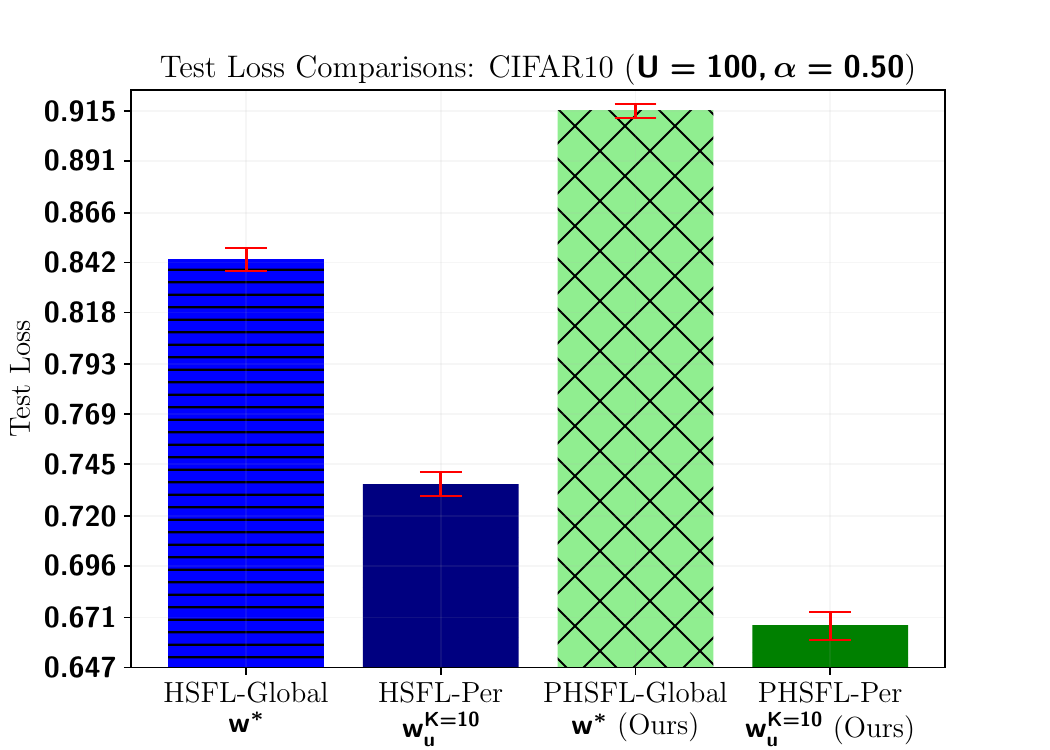}
    \caption{Test Loss Comparisons}
    \label{testLoss_alpha_0.5}
\end{subfigure}
\caption{Test performance comparisons on CIFAR$10$ across $U=100$ users, when $\mathrm{Dir}(\pmb{\alpha}=\mathbf{0.5)}$\bblue{: $\mathbf{w}^{*}$ and $\mathbf{w}_u^{K=10}$ represent the global trained model and fine-tuned personalized model, respectively}}
\label{testPerformance_alpha_0.5}
\end{figure}

\subsection{Performance Comparisons}
\noindent
In order to evaluate the performance, we adopt a \ac{hsfl} learning baseline based on \cite{khan2024joint,liao2024parallelsfl}.
More specifically, we assume the same system configurations, data distributions, and perfect communication between the clients and \acp{es} as in our proposed \ac{phsfl} algorithm.
Moreover, we also use a centralized \ac{sgd} baseline to get the performance upper bound, assuming a \emph{Genie} has access to all data samples. 
The \ac{sgd} algorithm iterates over the entire dataset (using $[\bigcup_{u\in \mathcal{U}} |\mathcal{D}_u|]/N$ mini-batches) for $R=100$ epochs.

First, we investigate the generalization performance with the globally trained model $\mathbf{w}^{t}$. 
Intuitively, as only the body parts of the model get trained with respect to a \emph{randomly initialized classifier}, the classification results may not be as accurate as when the entire model gets trained. 
This is due to the fact that the head/classifier does not update the initial weights as training progresses, as shown in (\ref{headUpdate}).
However, recall that our sole interest lies in having personalized models for the clients. 
Our simulation results in Fig. \ref{gRoundVstestPerformance_alpha_0.5} also validate this intuition. 
We observe a small gap between the generalized model performances from these two algorithms. 
For example, upon finishing $R=100$ global rounds, \ac{phsfl} and \ac{hsfl} have $0.7049 \pm 0.0039$ and $0.7163 \pm 0.0015$ test accuracies, respectively.
Besides, the test losses are $0.9104 \pm 0.0053$ and $0.8432 \pm 0.0053$, respectively, for \ac{phsfl} and \ac{hsfl}.
Moreover, both of these two algorithms have some clear performance deviations from the centralized \ac{sgd} baseline, which is expected since they do not have access to the entire training dataset.

We now investigate the personalized models that are fine-tuned from the globally trained models with \ac{phsfl} and \ac{hsfl} algorithms.
In particular, we fine-tune the classifier/head of the model using $K=10$ \ac{sgd} steps following the steps described in Section \ref{personalizationSteps}.
Note that since a higher skewness in the data distribution yields degraded generalized performance, model personalization shall benefit the most in such cases. 
Regardless of the data skewness, we expect \ac{phsfl} to perform better since it only learns the feature representations using the same classifier: since the features have similarities, severe data heterogeneity may not influence the lower model blocks. 
However, since \ac{hsfl} updates all model blocks, data heterogeneity may affect the weights of the classifier/head, which may complicate fine-tuning the global model with a small $K$.

To understand the impact of non-IID data distributions, we consider two cases: $\mathrm{Dir}(\pmb{\alpha}=\mathbf{0.1})$ and $\mathrm{Dir}(\pmb{\alpha}=\mathbf{0.5})$. 
Note that smaller $\pmb{\alpha}$ means more skewed data distribution. Our simulation results show the above trends in Figs. \ref{testPerformance_alpha_0.1} and \ref{testPerformance_alpha_0.5}.
When $\pmb{\alpha}=\mathbf{0.1}$, we observe that the mean test accuracies from the global models are $0.6606 \pm 0.0052$ and $0.64458 \pm 0.0058$, respectively, for \ac{hsfl} and \ac{phsfl}, while the corresponding test losses are $1.0049 \pm 0.0046$ and $1.0975 \pm 0.0073$.
However, the personalized mean test accuracies are $0.7958 \pm 0.0075$ and $0.8708 \pm 0.0012$, while the test losses are $0.6735 \pm 0.0138$ and $0.4721 \pm 0.0118$, respectively, for \ac{hsfl} and \ac{phsfl}.
Therefore, our proposed \ac{phsfl} solution has $9.43\%$ and $42.68\%$ improvement in test accuracy and test loss over \ac{hsfl}.
Besides, when $\pmb{\alpha}=\mathbf{0.5}$, the personalized models have $4.75\%$ and $9.03\%$ mean test accuracy improvement over the generalized global trained models, respectively, for \ac{hsfl} and \ac{phsfl}.

\acresetall
\section{Conclusions}
\noindent
We proposed a \ac{phsfl} algorithm that leverages \ac{sl} and \ac{hfl} and enables training an \ac{ml} model with resource-constrained distributed clients in heterogeneous wireless networks.
Our proposed algorithm trains the model to learn feature representations with a fixed classifier that never gets trained during the training phase. 
The fine-tuned classifier with the remaining trained global model blocks shows significant personalization performance improvement over other existing algorithms.


\bibliographystyle{IEEEtran}
\bibliography{reference}

 \newpage 
 \clearpage
 \appendices 
 \onecolumn 
 \section{Convergence Analysis}
 \subsection{Assumptions}
 \noindent
 We make the following standard assumptions \cite{wang2020tackling, pervej2024hierarchical,wang2022demystifying, pervej2025resource, ye2023feddisco} 
 \setcounter{Assumption}{0}
 \begin{Assumption}[Smoothness] 
 \label{AsumpSmoothNess}
     The loss functions are $\beta$-Lipschitz smooth, i.e., for some $\beta>0$, $\Vert \nabla f(\mathbf{w}) - \nabla f(\mathbf{w}') \Vert \leq \beta \Vert \mathbf{w} - \mathbf{w}' \Vert$, where $\Vert \cdot \Vert$ is the $L_2$ norm, for all loss functions.
 \end{Assumption}

 \begin{Assumption}[Unbiased gradient with bounded variance]
 \label{AsumpUnbiased}
     The mini-batch stochastic gradient calculated using client's randomly sampled mini-batch $\zeta$ is unbiased, i.e., $\mathbb{E}_{\zeta \sim \mathcal{D}_u} \left[g_u(\mathbf{w}) \right] = \nabla f_u (\mathbf{w})$, where $\mathbb{E} [\cdot]$ is the expectation operator.
     Besides, the variance of the gradients is bounded, i.e., $ \mathbb{E}_{\zeta \sim \mathcal{D}_u} \left[ \Vert \nabla f_u (\mathbf{w}) - g_u(\mathbf{w}) \Vert^2 \right] \leq \sigma^2 $, for some $\sigma \geq 0$ and all $u \in \mathcal{U}$.
 \end{Assumption}

 \begin{Assumption}[Bounded gradient divergence]
 \label{AsumpBoundedDivergence}
     The divergence between (a) local and \ac{es} loss functions and (b) \ac{es} and global loss functions are bounded as
     \begin{align}
         \sum\nolimits_{u \in \mathcal{U}_b} \alpha_u \Vert \nabla f_u (\mathbf{w}) - \nabla f_b (\mathbf{w}) \Vert^2 \leq \epsilon_0^2, \\
         \sum\nolimits_{b=1}^B \alpha_b \Vert \nabla f_b (\mathbf{w}) - \nabla f (\mathbf{w}) \Vert^2 \leq \epsilon_1^2, 
     \end{align}
     for some $\epsilon_0 \leq 0$, $\epsilon_1 \leq 0$ and all $u$ and $b$.
 \end{Assumption}

 Additionally, since the global and edge models do not exist in every \ac{sgd} step $t$, following standard practice \cite{pervej2024hierarchical,wang2022demystifying}, we assume that virtual copies of these models exist, denoted by $\bar{\mathbf{w}}^t$ and $\bar{\mathbf{w}}_b^t$, respectively.
 Moreover, the above assumptions also apply to these virtual models.

 \subsection{Convergence Analysis}
 \setcounter{Theorem}{0}
 \begin{Theorem}
 Suppose the above assumptions hold. Then if the learning rate satisfies $\eta < \frac{1}{2\sqrt{5} \beta \kappa_1 \kappa_0}$, the average global gradient norm is upper bounded by
 \begin{align}
 \label{covBound_App}
     \frac{1}{T} \sum_{t=0}^{T-1} \mathbb{E} \left[\left\Vert \nabla f (\bar{\mathbf{w}}^t) \right\Vert^2 \right] 
     &\leq \frac{2 \left(\mathbb{E} \left[f \left( \bar{\mathbf{w}}^0 \right) \right] - \mathbb{E} \left[f \left( \bar{\mathbf{w}}^{T} \right) \right]\right)}{\eta T} + \beta \eta \sigma^2 \sum_{b=0}^{B-1} \alpha_b^2 \sum_{u \in \mathcal{U}_b} \alpha_u^2 + \Gamma_0\sigma^2 + \Gamma_1\sigma^2 + 12 \beta^2 \epsilon_0^2 \eta^2 \kappa_0^2 + \nonumber \\
     &\qquad 20 \beta^2 \epsilon_1^2 \eta^2 \kappa_1^2 \kappa_0^2 + 240 \epsilon_0^2 \kappa_1^2 \beta^4 \eta^4 \kappa_0^4,
 \end{align}
 where $\Gamma_0 \coloneqq  4\beta^2 \eta^2 \kappa_0^2 - 4 \beta^2 \eta^2 \kappa_0^2  \sum_{b=0}^{B-1} \alpha_b \sum_{u \in \mathcal{U}_b} \alpha_{u}^2 + 80 \kappa_1^2 \beta^4 \eta^4 \kappa_0^4$, 
 $\Gamma_1 \coloneqq 4 \kappa_1 \kappa_0 \beta^2 \eta^2 \sum_{b=0}^{B-1} \alpha_b \sum_{u \in \mathcal{U}_b} \alpha_{u}^2 - 4 \kappa_1 \kappa_0 \beta^2 \eta^2 \sum_{b=0}^{B-1} \alpha_{b}^2 \sum_{u\in \mathcal{U}_{b}} \alpha_u^2 - 80 \kappa_1^2 \beta^4 \eta^4 \kappa_0^4 \sum_{b=0}^{B-1} \alpha_b \sum_{u \in \mathcal{U}_b} \alpha_{u}^2$. 
 \end{Theorem}

 \begin{proof}
 \noindent
 Although the global model $\mathbf{w}^t$ and the \ac{es} model $\mathbf{w}_b^t$ does not exists in all $t$, for our convergence analysis, we assume that virtual copies of these models are available.
 Denote the global and \ac{es} virtual models by $\Bar{\mathbf{w}}^t$ and $\Bar{\mathbf{w}}_b^t$.

 \bigskip
 The aggregation rules with these virtual models are as
 \begin{align}
     \bar{\mathbf{w}}_b^{t+1} &= \sum_{u \in \mathcal{U}_b} \alpha_u \mathbf{w}_u^{t+1} = \sum_{u \in \mathcal{U}_b} \alpha_u \left[\mathbf{w}_u^{t} - \eta g_{u} (\mathbf{w}_{u}^t) \right] = \bar{\mathbf{w}}_b^t - \eta \sum_{u \in \mathcal{U}_b} \alpha_u g_{u} (\mathbf{w}_{u}^t), \\ 
     \bar{\mathbf{w}}^{t+1} &= \sum_{b=0}^{B-1} \alpha_b \bar{\mathbf{w}}_b^{t+1} = \sum_{b=0}^{B-1} \alpha_b \left[\bar{\mathbf{w}}_b^t - \eta \sum_{u \in \mathcal{U}_b} \alpha_u g_u(\mathbf{w}_{u}^t) \right] = \bar{\mathbf{w}}^t - \eta \sum_{b=0}^{B-1} \alpha_b \sum_{u \in \mathcal{U}_b} \alpha_u g_{u} (\mathbf{w}_{u}^t), \label{globalAggVir} \\ 
 \end{align}
 where $\mathbf{w}_{u}^t = [\mathbf{w}_{u,0}^t; \mathbf{w}_{u,1}^t]$.

 Using the aggregation rule in (\ref{globalAggVir}), we write
 \begin{align}
     f(\bar{\mathbf{w}}^{t+1}) 
     &= f \left( \bar{\mathbf{w}}^t - \eta \sum_{b=0}^{B-1} \alpha_b \sum_{u \in \mathcal{U}_b} \alpha_u g_{u} (\mathbf{w}_{u}^t) \right) \nonumber\\
     &\overset{(a)}{\leq} f \left( \bar{\mathbf{w}}^t \right) + \left< \nabla f (\bar{\mathbf{w}}^t), -\eta \sum_{b=0}^{B-1} \alpha_b \sum_{u \in \mathcal{U}_b} \alpha_u g_{u} (\mathbf{w}_{u}^t) \right> + \frac{\beta}{2} \left \Vert -\eta \sum_{b=0}^{B-1} \alpha_b \sum_{u \in \mathcal{U}_b} \alpha_u g_{u} (\mathbf{w}_{u}^t) \right\Vert^2 \nonumber\\
     &=f \left( \bar{\mathbf{w}}^t \right) - \eta \left< \nabla f (\bar{\mathbf{w}}^t), \sum_{b=0}^{B-1} \alpha_b \sum_{u \in \mathcal{U}_b} \alpha_u g_{u} (\mathbf{w}_{u}^t) \right> + \frac{\beta \eta^2}{2} \left \Vert \sum_{b=0}^{B-1} \alpha_b \sum_{u \in \mathcal{U}_b} \alpha_u g_{u} (\mathbf{w}_{u}^t) \right\Vert^2, 
 \end{align}
 where $(a)$ stems from $\beta$-Lipschitz smoothness, i.e., Assumption \ref{AsumpSmoothNess}.

 Taking expectation over both sides,
 \begin{align}
 \label{convExpectationTerm}
     \mathbb{E} \left[ f \left(\bar{\mathbf{w}}^{t+1}\right) \right] = \mathbb{E} \left[f \left( \bar{\mathbf{w}}^t \right) \right] - \eta \mathbb{E} \left[\left< \nabla f (\bar{\mathbf{w}}^t), \sum_{b=0}^{B-1} \alpha_b \sum_{u \in \mathcal{U}_b} \alpha_u g_{u} (\mathbf{w}_{u}^t) \right> \right] + \frac{\beta \eta^2}{2} \mathbb{E} \left[ \left \Vert \sum_{b=0}^{B-1} \alpha_b \sum_{u \in \mathcal{U}_b} \alpha_u g_{u} (\mathbf{w}_{u}^t) \right\Vert^2\right]. 
 \end{align}

 Now, we simplify the second term of (\ref{convExpectationTerm}) as
 \begin{align}
 \label{innerProdTerm}
     &- \eta \mathbb{E} \left[\left< \nabla f (\bar{\mathbf{w}}^t), \sum_{b=0}^{B-1} \alpha_b \sum_{u \in \mathcal{U}_b} \alpha_u g_{u} (\mathbf{w}_{u}^t) \right> \right] \nonumber \\
     &=- \eta  \left< \nabla f (\bar{\mathbf{w}}^t), \sum_{b=0}^{B-1} \alpha_b \sum_{u \in \mathcal{U}_b} \alpha_u \mathbb{E} \left[g_{u} (\mathbf{w}_{u}^t) \right]\right> \nonumber\\
     &\overset{(a)}{=}- \eta  \left< \nabla f (\bar{\mathbf{w}}^t), \sum_{b=0}^{B-1} \alpha_b \sum_{u \in \mathcal{U}_b} \alpha_u \nabla f_{u} (\mathbf{w}_{u}^t) \right> \nonumber\\
     &\overset{(b)}{=} \frac{-\eta}{2} \left[ \left\Vert \nabla f (\bar{\mathbf{w}}^t) \right\Vert^2 + \left\Vert \sum_{b=0}^{B-1} \alpha_b \sum_{u \in \mathcal{U}_b} \alpha_u \nabla f_{u} (\mathbf{w}_{u}^t) \right\Vert^2 - \left\Vert \nabla f (\bar{\mathbf{w}}^t) - \sum_{b=0}^{B-1} \alpha_b \sum_{u \in \mathcal{U}_b} \alpha_u \nabla f_{u} (\mathbf{w}_{u}^t)\right\Vert^2\right] \nonumber\\
     &= \frac{\eta}{2}\left\Vert \nabla f (\bar{\mathbf{w}}^t) - \sum_{b=0}^{B-1} \alpha_b \sum_{u \in \mathcal{U}_b} \alpha_u \nabla f_{u} (\mathbf{w}_{u}^t)\right\Vert^2 - \frac{\eta}{2} \left\Vert \nabla f (\bar{\mathbf{w}}^t) \right\Vert^2 - \frac{\eta}{2} \left\Vert \sum_{b=0}^{B-1} \alpha_b \sum_{u \in \mathcal{U}_b} \alpha_u \nabla f_{u} (\mathbf{w}_{u}^t) \right\Vert^2 
 \end{align}
 where $(a)$ and $(b)$ are true due to the unbiased stochastic gradients assumption and the fact that $\Vert \mathbf{a} - \mathbf{b} \Vert^2 = \Vert \mathbf{a} \Vert^2 + \Vert \mathbf{b} \Vert^2 - 2<\mathbf{a},\mathbf{b}> $.

 Now, we simplify the third term of (\ref{convExpectationTerm}) as
 \begin{align}
 \label{third_Term}
     &\frac{\beta \eta^2}{2} \mathbb{E} \left[ \left \Vert \sum_{b=0}^{B-1} \alpha_b \sum_{u \in \mathcal{U}_b} \alpha_u g_{u} (\mathbf{w}_{u}^t) \right\Vert^2\right] \nonumber\\
     &\overset{(a)}{=} \frac{\beta \eta^2}{2} \left[ \mathbb{E} \left[ \left \Vert \sum_{b=0}^{B-1} \alpha_b \sum_{u \in \mathcal{U}_b} \alpha_u g_{u} (\mathbf{w}_{u}^t) - \mathbb{E} \left[\sum_{b=0}^{B-1} \alpha_b \sum_{u \in \mathcal{U}_b} \alpha_u g_{u} (\mathbf{w}_{u}^t) \right] \right\Vert^2\right] + \left(\mathbb{E} \left[\sum_{b=0}^{B-1} \alpha_b \sum_{u \in \mathcal{U}_b} \alpha_u g_{u} (\mathbf{w}_{u}^t) \right]\right)^2 \right] \nonumber\\
     &\overset{(b)}{=} \frac{\beta \eta^2}{2} \left[ \mathbb{E} \left[ \left \Vert \sum_{b=0}^{B-1} \alpha_b \sum_{u \in \mathcal{U}_b} \alpha_u g_{u} (\mathbf{w}_{u}^t) - \sum_{b=0}^{B-1} \alpha_b \sum_{u \in \mathcal{U}_b} \alpha_u \nabla f_{u} (\mathbf{w}_{u}^t) \right\Vert^2 \right] + \left( \sum_{b=0}^{B-1} \alpha_b \sum_{u \in \mathcal{U}_b} \alpha_u \nabla f_{u} (\mathbf{w}_{u}^t) \right)^2 \right] \nonumber\\
     &=\frac{\beta \eta^2}{2} \left[ \mathbb{E} \left[ \left \Vert \sum_{b=0}^{B-1} \alpha_b \sum_{u \in \mathcal{U}_b} \alpha_u \left( g_{u} (\mathbf{w}_{u}^t) - \nabla f_{u} (\mathbf{w}_{u}^t)\right) \right\Vert^2 \right] + \left\Vert \sum_{b=0}^{B-1} \alpha_b \sum_{u \in \mathcal{U}_b} \alpha_u \nabla f_{u} (\mathbf{w}_{u}^t) \right\Vert^2 \right] \nonumber\\
     &\overset{(c)}{=} \frac{\beta \eta^2}{2} \left[ \sum_{b=0}^{B-1} \alpha_b^2 \sum_{u \in \mathcal{U}_b} \alpha_u^2 \mathbb{E} \left[ \left \Vert g_{u} (\mathbf{w}_{u}^t) - \nabla f_{u} (\mathbf{w}_{u}^t) \right\Vert^2 \right] + \left\Vert \sum_{b=0}^{B-1} \alpha_b \sum_{u \in \mathcal{U}_b} \alpha_u \nabla f_{u} (\mathbf{w}_{u}^t) \right\Vert^2 \right] \nonumber\\
     &\overset{(d)}{\leq} \frac{\beta \eta^2}{2} \left[ \sum_{b=0}^{B-1} \alpha_b^2 \sum_{u \in \mathcal{U}_b} \alpha_u^2 \cdot \sigma^2 + \left\Vert \sum_{b=0}^{B-1} \alpha_b \sum_{u \in \mathcal{U}_b} \alpha_u \nabla f_{u} (\mathbf{w}_{u}^t) \right\Vert^2 \right] \nonumber\\
     &=\frac{\beta \eta^2 \sigma^2}{2} \sum_{b=0}^{B-1} \alpha_b^2 \sum_{u \in \mathcal{U}_b} \alpha_u^2 + \frac{\beta \eta^2}{2} \left\Vert \sum_{b=0}^{B-1} \alpha_b \sum_{u \in \mathcal{U}_b} \alpha_u \nabla f_{u} (\mathbf{w}_{u}^t) \right\Vert^2, 
 \end{align}
 where $(a)$ is from the definition of variance; $(b)$ stems from Assumption \ref{AsumpUnbiased}; $(c)$ stems from the independence of stochastic gradients and Assumption \ref{AsumpUnbiased}; $(d)$ is true due to the bounded variance of the gradients assumption.

 Now, plugging (\ref{innerProdTerm}) and (\ref{third_Term}) into (\ref{convExpectationTerm}), we get
 \begin{align}
 \label{convExpectationTerm_1}
     \mathbb{E} \left[ f \left(\bar{\mathbf{w}}^{t+1}\right) \right] \leq 
     & \mathbb{E} \left[f \left( \bar{\mathbf{w}}^t \right) \right] + \frac{\eta}{2}\left\Vert \nabla f (\bar{\mathbf{w}}^t) - \sum_{b=0}^{B-1} \alpha_b \sum_{u \in \mathcal{U}_b} \alpha_u \nabla f_{u} (\mathbf{w}_{u}^t)\right\Vert^2 - \frac{\eta}{2} \left\Vert \nabla f (\bar{\mathbf{w}}^t) \right\Vert^2 - \frac{\eta}{2} \left\Vert \sum_{b=0}^{B-1} \alpha_b \sum_{u \in \mathcal{U}_b} \alpha_u \nabla f_{u} (\mathbf{w}_{u}^t) \right\Vert^2 
     + \nonumber\\
     &\frac{\beta \eta^2 \sigma^2}{2} \sum_{b=0}^{B-1} \alpha_b^2 \sum_{u \in \mathcal{U}_b} \alpha_u^2 + \frac{\beta \eta^2}{2} \left\Vert \sum_{b=0}^{B-1} \alpha_b \sum_{u \in \mathcal{U}_b} \alpha_u \nabla f_{u} (\mathbf{w}_{u}^t) \right\Vert^2 \nonumber\\
     &= \mathbb{E} \left[f \left( \bar{\mathbf{w}}^t \right) \right] + \frac{\beta \eta^2 \sigma^2}{2} \sum_{b=0}^{B-1} \alpha_b^2 \sum_{u \in \mathcal{U}_b} \alpha_u^2 + \frac{\eta}{2}\left\Vert \nabla f (\bar{\mathbf{w}}^t) - \sum_{b=0}^{B-1} \alpha_b \sum_{u \in \mathcal{U}_b} \alpha_u \nabla f_{u} (\mathbf{w}_{u}^t)\right\Vert^2 - \frac{\eta}{2} \left\Vert \nabla f (\bar{\mathbf{w}}^t) \right\Vert^2 \nonumber\\
     & - \frac{\eta}{2} \left(1 - \beta\eta \right)\left\Vert \sum_{b=0}^{B-1} \alpha_b \sum_{u \in \mathcal{U}_b} \alpha_u \nabla f_{u} (\mathbf{w}_{u}^t) \right\Vert^2 
 \end{align}

 When $\eta \leq \frac{1}{\beta}$, we have $(1 - \beta\eta) \geq 0$. 
 As $\Vert \cdot \Vert$ is non-negative and we are after a lower-bound, we drop the last term from (\ref{convExpectationTerm_1}).
 Then, rearranging the terms, dividing both sides by $\frac{\eta}{2}$, and taking expectation on both sides, we get 
 \begin{align}
 \label{convBound_Interm}
     \mathbb{E} \left[\left\Vert \nabla f (\bar{\mathbf{w}}^t) \right\Vert^2 \right] \leq
     & \frac{2 \left(\mathbb{E} \left[f \left( \bar{\mathbf{w}}^t \right) \right] - \mathbb{E} \left[f \left( \bar{\mathbf{w}}^{t+1} \right) \right]\right)}{\eta} + \beta \eta \sigma^2 \sum_{b=0}^{B-1} \alpha_b^2 \sum_{u \in \mathcal{U}_b} \alpha_u^2 + \mathbb{E} \left[\left\Vert \nabla f (\bar{\mathbf{w}}^t) - \sum_{b=0}^{B-1} \alpha_b \sum_{u \in \mathcal{U}_b} \alpha_u \nabla f_{u} (\mathbf{w}_{u}^t)\right\Vert^2 \right].
 \end{align}

 We now simplify the last term of (\ref{convBound_Interm}) as

 \begin{align}
 \label{convBound_Interm_0}
     &\mathbb{E} \left[ \left\Vert \nabla f (\bar{\mathbf{w}}^t) - \sum_{b=0}^{B-1} \alpha_b \sum_{u \in \mathcal{U}_b} \alpha_u \nabla f_{u} (\mathbf{w}_{u}^t)\right\Vert^2 \right] \nonumber\\
     &\overset{(a)}{=} \mathbb{E} \left[ \left\Vert \sum_{b=0}^{B-1} \alpha_b \sum_{u \in \mathcal{U}_b} \alpha_u \nabla f_u (\bar{\mathbf{w}}^t) - \sum_{b=0}^{B-1} \alpha_b \sum_{u \in \mathcal{U}_b} \alpha_u \nabla f_{u} (\mathbf{w}_{u}^t)\right\Vert^2 \right] \nonumber\\
     &\overset{(b)}{\leq} \sum_{b=0}^{B-1} \alpha_b \sum_{u \in \mathcal{U}_b} \alpha_u ~ \mathbb{E} \left[ \left\Vert \nabla f_u (\bar{\mathbf{w}}^t) - \nabla f_{u} (\mathbf{w}_{u}^t)\right\Vert^2 \right] \nonumber\\
     &= \sum_{b=0}^{B-1} \alpha_b \sum_{u \in \mathcal{U}_b} \alpha_u ~ \mathbb{E} \left[ \left\Vert \nabla f_u (\bar{\mathbf{w}}^t) - \nabla f_u (\bar{\mathbf{w}}_b^t) + \nabla f_u (\bar{\mathbf{w}}_b^t) - \nabla f_{u} (\mathbf{w}_{u}^t)\right\Vert^2 \right] \nonumber\\
     &\overset{(c)}{\leq} 2 \sum_{b=0}^{B-1} \alpha_b \sum_{u \in \mathcal{U}_b} \alpha_u ~ \mathbb{E} \left[ \left\Vert \nabla f_u (\bar{\mathbf{w}}^t) - \nabla f_u (\bar{\mathbf{w}}_b^t) \right\Vert^2 \right] + 2 \sum_{b=0}^{B-1} \alpha_b \sum_{u \in \mathcal{U}_b} \alpha_u ~ \mathbb{E} \left[ \left\Vert \nabla f_u (\bar{\mathbf{w}}_b^t) - \nabla f_{u} (\mathbf{w}_{u}^t)\right\Vert^2 \right] \nonumber\\
     &\overset{(d)}{\leq} 2 \beta^2 \sum_{b=0}^{B-1} \alpha_b \sum_{u \in \mathcal{U}_b} \alpha_u ~ \mathbb{E} \left[ \left\Vert \bar{\mathbf{w}}^t - \bar{\mathbf{w}}_b^t \right\Vert^2 \right] + 2 \beta^2 \sum_{b=0}^{B-1} \alpha_b \sum_{u \in \mathcal{U}_b} \alpha_u ~ \mathbb{E} \left[ \left\Vert \bar{\mathbf{w}}_b^t - \mathbf{w}_{u}^t \right\Vert^2 \right] \nonumber\\
     &= 2 \beta^2 \sum_{b=0}^{B-1} \alpha_b ~ \mathbb{E} \left[ \left\Vert \bar{\mathbf{w}}_b^t - \mathbf{w}_{u}^t \right\Vert^2 \right] + 2 \beta^2 \sum_{b=0}^{B-1} \alpha_b \sum_{u \in \mathcal{U}_b} \alpha_u ~ \mathbb{E} \left[ \left\Vert \bar{\mathbf{w}}^t - \bar{\mathbf{w}}_b^t \right\Vert^2 \right] 
 \end{align}
 where $(a)$ comes from the fact that $f (\mathbf{w}) = \sum_{b=0}^{B-1} \alpha_b \sum_{u \in \mathcal{U}_b} \alpha_u f_u (\mathbf{w})$ by definition.
 Besides, $(b)$ stems from the convexity of $\Vert \cdot \Vert$ and Jensen inequality.
 In $(c)$, we use $\Vert \sum_{i=1}^I \mathbf{a}_i \Vert^2 \leq I  \sum_{i=1}^I \Vert \mathbf{a}_i \Vert^2 $.
 The last inequality appears from Assumption \ref{AsumpSmoothNess}.

 Plugging (\ref{convBound_Interm_0}) into (\ref{convBound_Interm}), we get
 \begin{align}
     \mathbb{E} \left[\left\Vert \nabla f (\bar{\mathbf{w}}^t) \right\Vert^2 \right] \leq
     & \frac{2 \left(\mathbb{E} \left[f \left( \bar{\mathbf{w}}^t \right) \right] - \mathbb{E} \left[f \left( \bar{\mathbf{w}}^{t+1} \right) \right]\right)}{\eta} + \beta \eta \sigma^2 \sum_{b=0}^{B-1} \alpha_b^2 \sum_{u \in \mathcal{U}_b} \alpha_u^2 + 2 \beta^2 \sum_{b=0}^{B-1} \alpha_b ~ \mathbb{E} \left[ \left\Vert \bar{\mathbf{w}}_b^t - \mathbf{w}_{u}^t \right\Vert^2 \right] + \nonumber\\
     &2 \beta^2 \sum_{b=0}^{B-1} \alpha_b \sum_{u \in \mathcal{U}_b} \alpha_u ~ \mathbb{E} \left[ \left\Vert \bar{\mathbf{w}}^t - \bar{\mathbf{w}}_b^t \right\Vert^2 \right]
 \end{align}

 Now, if $\eta \leq \frac{1}{2\sqrt{5} \beta \kappa_1\kappa_0}$, using Lemma \ref{lemma1App} and Lemma \ref{lemma2App}, we write
 \begin{align}
     \mathbb{E} \left[\left\Vert \nabla f (\bar{\mathbf{w}}^t) \right\Vert^2 \right] \leq
     & \frac{2 \left(\mathbb{E} \left[f \left( \bar{\mathbf{w}}^t \right) \right] - \mathbb{E} \left[f \left( \bar{\mathbf{w}}^{t+1} \right) \right]\right)}{\eta} + \beta \eta \sigma^2 \sum_{b=0}^{B-1} \alpha_b^2 \sum_{u \in \mathcal{U}_b} \alpha_u^2 +  4\beta^2 \eta^2 \kappa_0^2 \sigma^2 + 12 \beta^2 \epsilon_0^2 \eta^2 \kappa_0^2 - 4 \beta^2 \eta^2 \kappa_0^2 \sigma^2 \sum_{b=0}^{B-1} \alpha_b \sum_{u \in \mathcal{U}_b} \alpha_{u}^2 + \nonumber\\
     &20 \beta^2 \epsilon_1^2 \eta^2 \kappa_1^2 \kappa_0^2 + 80 \kappa_1^2 \sigma^2 \beta^4 \eta^4 \kappa_0^4 + 240 \epsilon_0^2 \kappa_1^2 \beta^4 \eta^4 \kappa_0^4 + 4 \kappa_1 \kappa_0 \beta^2 \eta^2 \sigma^2 \sum_{b=0}^{B-1} \alpha_b \sum_{u \in \mathcal{U}_b} \alpha_{u}^2 - \nonumber\\
     &4 \kappa_1 \kappa_0 \beta^2 \eta^2 \sigma^2 \sum_{b=0}^{B-1} \alpha_{b}^2 \sum_{u\in \mathcal{U}_{b}} \alpha_u^2 - 80 \kappa_1^2 \sigma^2 \beta^4 \eta^4 \kappa_0^4 \sum_{b=0}^{B-1} \alpha_b \sum_{u \in \mathcal{U}_b} \alpha_{u}^2 \nonumber\\
     &=\frac{2 \left(\mathbb{E} \left[f \left( \bar{\mathbf{w}}^t \right) \right] - \mathbb{E} \left[f \left( \bar{\mathbf{w}}^{t+1} \right) \right]\right)}{\eta} + \beta \eta \sigma^2 \sum_{b=0}^{B-1} \alpha_b^2 \sum_{u \in \mathcal{U}_b} \alpha_u^2 +  4\beta^2 \eta^2 \kappa_0^2 \sigma^2 - 4 \beta^2 \eta^2 \kappa_0^2 \sigma^2 \sum_{b=0}^{B-1} \alpha_b \sum_{u \in \mathcal{U}_b} \alpha_{u}^2 + \nonumber\\
     &80 \kappa_1^2 \sigma^2 \beta^4 \eta^4 \kappa_0^4 + 4 \kappa_1 \kappa_0 \beta^2 \eta^2 \sigma^2 \sum_{b=0}^{B-1} \alpha_b \sum_{u \in \mathcal{U}_b} \alpha_{u}^2 - 4 \kappa_1 \kappa_0 \beta^2 \eta^2 \sigma^2 \sum_{b=0}^{B-1} \alpha_{b}^2 \sum_{u\in \mathcal{U}_{b}} \alpha_u^2 - 80 \kappa_1^2 \sigma^2 \beta^4 \eta^4 \kappa_0^4 \sum_{b=0}^{B-1} \alpha_b \sum_{u \in \mathcal{U}_b} \alpha_{u}^2 \nonumber\\
     &12 \beta^2 \epsilon_0^2 \eta^2 \kappa_0^2 +20 \beta^2 \epsilon_1^2 \eta^2 \kappa_1^2 \kappa_0^2 + 240 \epsilon_0^2 \kappa_1^2 \beta^4 \eta^4 \kappa_0^4 \nonumber\\
     &=\frac{2 \left(\mathbb{E} \left[f \left( \bar{\mathbf{w}}^t \right) \right] - \mathbb{E} \left[f \left( \bar{\mathbf{w}}^{t+1} \right) \right]\right)}{\eta} + \beta \eta \sigma^2 \sum_{b=0}^{B-1} \alpha_b^2 \sum_{u \in \mathcal{U}_b} \alpha_u^2 + \Gamma_0\sigma^2 + \Gamma_1\sigma^2 + 12 \beta^2 \epsilon_0^2 \eta^2 \kappa_0^2 + \nonumber \\
     &20 \beta^2 \epsilon_1^2 \eta^2 \kappa_1^2 \kappa_0^2 + 240 \epsilon_0^2 \kappa_1^2 \beta^4 \eta^4 \kappa_0^4,
 \end{align}
 where $\Gamma_0 \coloneqq  4\beta^2 \eta^2 \kappa_0^2 - 4 \beta^2 \eta^2 \kappa_0^2  \sum_{b=0}^{B-1} \alpha_b \sum_{u \in \mathcal{U}_b} \alpha_{u}^2 + 80 \kappa_1^2 \beta^4 \eta^4 \kappa_0^4$, 
 $\Gamma_1 \coloneqq 4 \kappa_1 \kappa_0 \beta^2 \eta^2 \sum_{b=0}^{B-1} \alpha_b \sum_{u \in \mathcal{U}_b} \alpha_{u}^2 - 4 \kappa_1 \kappa_0 \beta^2 \eta^2 \sum_{b=0}^{B-1} \alpha_{b}^2 \sum_{u\in \mathcal{U}_{b}} \alpha_u^2 - 80 \kappa_1^2 \beta^4 \eta^4 \kappa_0^4 \sum_{b=0}^{B-1} \alpha_b \sum_{u \in \mathcal{U}_b} \alpha_{u}^2$.

 To that end, we get the following by averaging over time
 \begin{align}
 \label{convBound_Interm_1}
     \frac{1}{T} \sum_{t=0}^{T-1} \mathbb{E} \left[\left\Vert \nabla f (\bar{\mathbf{w}}^t) \right\Vert^2 \right] 
     &\leq \frac{2 \left(\mathbb{E} \left[f \left( \bar{\mathbf{w}}^0 \right) \right] - \mathbb{E} \left[f \left( \bar{\mathbf{w}}^{T} \right) \right]\right)}{\eta T} + \beta \eta \sigma^2 \sum_{b=0}^{B-1} \alpha_b^2 \sum_{u \in \mathcal{U}_b} \alpha_u^2 + \Gamma_0\sigma^2 + \Gamma_1\sigma^2 + 12 \beta^2 \epsilon_0^2 \eta^2 \kappa_0^2 + \nonumber \\
     &\qquad 20 \beta^2 \epsilon_1^2 \eta^2 \kappa_1^2 \kappa_0^2 + 240 \epsilon_0^2 \kappa_1^2 \beta^4 \eta^4 \kappa_0^4.
 \end{align}
 \end{proof}

 \begin{Lemma}
 \label{lemma1App}
 For the first term of (\ref{convBound_Interm_0}), we have
 \begin{align}
     &2 \sum_{b=0}^{B-1} \alpha_b \sum_{u \in \mathcal{U}_b} \alpha_u ~ \mathbb{E} \left[ \left\Vert \bar{\mathbf{w}}_b^t - \mathbf{w}_{u}^t \right\Vert^2 \right] \nonumber\\
     &\qquad \qquad \leq 4 \eta^2 \kappa_0^2 \sigma^2 + 12 \epsilon_0^2 \eta^2 \kappa_0^2 - 4 \eta^2 \kappa_0^2 \sigma^2 \sum_{b=0}^{B-1} \alpha_b \sum_{u \in \mathcal{U}_b} \alpha_{u}^2 + 24 \beta^2 \eta^2 \kappa_0^2 \sum_{b=0}^{B-1} \alpha_b \sum_{u \in \mathcal{U}_b} \alpha_u \mathbb{E} \left[ \left\Vert \mathbf{w}_{u}^t - \bar{\mathbf{w}}_{b}^t \right\Vert^2 \right] .
 \end{align}
 Besides, if $\eta \leq \frac{1}{2\sqrt{3} \beta \kappa_0}$, we have
 \begin{align}
     &\sum_{b=0}^{B-1} \alpha_b \sum_{u \in \mathcal{U}_b} \alpha_u ~ \mathbb{E} \left[ \left\Vert \bar{\mathbf{w}}_b^t - \mathbf{w}_{u}^t \right\Vert^2 \right] \leq 2 \eta^2 \kappa_0^2 \sigma^2 + 6 \epsilon_0^2 \eta^2 \kappa_0^2 - 2 \eta^2 \kappa_0^2 \sigma^2 \sum_{b=0}^{B-1} \alpha_b \sum_{u \in \mathcal{U}_b} \alpha_{u}^2.
 \end{align}
 \end{Lemma}
 \begin{proof}
 Suppose $\bar{t}_0 \coloneqq t_2\kappa_1\kappa_0 + t_1\kappa_0$ and $t = \bar{t}_0 + t_0$, where $1 \leq t_0 \leq (\kappa_0 - 1)$. 
 Then we have 
 \begin{align}
 \label{lemma1MainEqn}
     &2 \sum_{b=0}^{B-1} \alpha_b \sum_{u \in \mathcal{U}_b} \alpha_u ~ \mathbb{E} \left[ \left\Vert \bar{\mathbf{w}}_b^t - \mathbf{w}_{u}^t \right\Vert^2 \right] \nonumber\\
     &= 2 \sum_{b=0}^{B-1} \alpha_b \sum_{u \in \mathcal{U}_b} \alpha_u ~ \mathbb{E} \left[ \left\Vert \mathbf{w}_b^{t_0} - \eta\sum_{u' \in \mathcal{U}_b} \alpha_{u'} \sum_{\tau=\bar{t}_0}^{t-1} g_{u'} (\mathbf{w}_{u'}^\tau) - \mathbf{w}_{u}^{\bar{t}_0} + \eta \sum_{\tau=\bar{t}_0}^{t-1} g_u (\mathbf{w}_u^{\tau}) \right\Vert^2 \right] \nonumber\\
     &\overset{(a)}{=} 2 \eta^2 \sum_{b=0}^{B-1} \alpha_b \sum_{u \in \mathcal{U}_b} \alpha_u ~ \mathbb{E} \left[ \left\Vert \sum_{\tau=\bar{t}_0}^{t-1} \left[\sum_{u' \in \mathcal{U}_b} \alpha_{u'}  g_{u'} (\mathbf{w}_{u'}^\tau) - g_u (\mathbf{w}_u^{\tau}) \right]\right\Vert^2 \right] \nonumber\\
     &= 2 \eta^2 \sum_{b=0}^{B-1} \alpha_b \sum_{u \in \mathcal{U}_b} \alpha_u ~ \mathbb{E} \Bigg[ \bigg\Vert \sum_{\tau=\bar{t}_0}^{t-1} \bigg[ \sum_{u' \in \mathcal{U}_b} \alpha_{u'} g_{u'} (\mathbf{w}_{u'}^\tau) - \sum_{u' \in \mathcal{U}_b} \alpha_{u'} \nabla f_{u'} (\mathbf{w}_{u'}^\tau) + \sum_{u' \in \mathcal{U}_b} \alpha_{u'} \nabla f_{u'} (\mathbf{w}_{u'}^\tau) - \nabla f_{u} (\mathbf{w}_{u}^\tau) + \nabla f_{u} (\mathbf{w}_{u}^\tau) - g_u (\mathbf{w}_u^{\tau}) \bigg\Vert^2 \Bigg] \nonumber \\
     &\overset{(b)}{\leq} 4 \eta^2 \sum_{b=0}^{B-1} \alpha_b \sum_{u \in \mathcal{U}_b} \alpha_u ~ \mathbb{E} \Bigg[ \bigg\Vert \sum_{\tau=\bar{t}_0}^{t-1} \bigg[ \left(\sum_{u' \in \mathcal{U}_b} \alpha_{u'} g_{u'} (\mathbf{w}_{u'}^\tau) - \sum_{u' \in \mathcal{U}_b} \alpha_{u'} \nabla f_{u'} (\mathbf{w}_{u'}^\tau) \right) + \left(\nabla f_{u} (\mathbf{w}_{u}^\tau) - g_u (\mathbf{w}_u^{\tau})\right) \bigg] \bigg\Vert^2 \Bigg] + \nonumber\\
     &\qquad 4 \eta^2 \sum_{b=0}^{B-1} \alpha_b \sum_{u \in \mathcal{U}_b} \alpha_u ~ \mathbb{E} \Bigg[ \bigg\Vert \sum_{\tau=\bar{t}_0}^{t-1} \left( \sum_{u' \in \mathcal{U}_b} \alpha_{u'} \nabla f_{u'} (\mathbf{w}_{u'}^\tau) - \nabla f_{u} (\mathbf{w}_{u}^\tau) \right) \bigg] \bigg\Vert^2 \Bigg] \nonumber \\
 \end{align}
 where $(a)$ stems from the fact that $\mathbf{w}_u^{\bar{t}_0} \gets \mathbf{w}_b^{\bar{t}_0}$ at $\bar{t}_0$ and the independence of the stochastic gradients.

 The first term is bounded  as
 \begin{align}
 \label{firstTermLemma1}
     &4 \eta^2 \sum_{b=0}^{B-1} \alpha_b \sum_{u \in \mathcal{U}_b} \alpha_u ~ \mathbb{E} \Bigg[ \bigg\Vert \sum_{\tau=\bar{t}_0}^{t-1} \bigg[ \left(\sum_{u' \in \mathcal{U}_b} \alpha_{u'} g_{u'} (\mathbf{w}_{u'}^\tau) - \sum_{u' \in \mathcal{U}_b} \alpha_{u'} \nabla f_{u'} (\mathbf{w}_{u'}^\tau) \right) + \left(\nabla f_{u} (\mathbf{w}_{u}^\tau) - g_u (\mathbf{w}_u^{\tau})\right) \bigg] \bigg\Vert^2 \Bigg] \nonumber\\
     &=4 \eta^2 \sum_{b=0}^{B-1} \alpha_b \sum_{u \in \mathcal{U}_b} \alpha_u ~ \mathbb{E} \Bigg[ \bigg\Vert \sum_{\tau=\bar{t}_0}^{t-1} \bigg[ \left(\nabla f_{u} (\mathbf{w}_{u}^\tau) - g_u (\mathbf{w}_u^{\tau})\right) - \sum_{u' \in \mathcal{U}_b} \alpha_{u'} \left( \nabla f_{u'} (\mathbf{w}_{u'}^\tau) - g_{u'} (\mathbf{w}_{u'}^\tau) \right) \bigg] \bigg\Vert^2 \Bigg] \nonumber\\
     &\overset{(a)}{=} 4 \eta^2 \sum_{b=0}^{B-1} \alpha_b \sum_{u \in \mathcal{U}_b} \alpha_u ~ \mathbb{E} \left[ \left\Vert \sum_{\tau=\bar{t}_0}^{t-1} \left(\nabla f_{u} (\mathbf{w}_{u}^\tau) - g_u (\mathbf{w}_u^{\tau})\right) \right\Vert^2 \right] - 4 \eta^2 \sum_{b=0}^{B-1} \alpha_b ~ \mathbb{E} \Bigg[ \bigg\Vert \sum_{\tau=\bar{t}_0}^{t-1} \sum_{u' \in \mathcal{U}_b} \alpha_{u'} \left( \nabla f_{u'} (\mathbf{w}_{u'}^\tau) - g_{u'} (\mathbf{w}_{u'}^\tau) \right) \bigg\Vert^2 \Bigg] \nonumber\\
     &\overset{(b)}{=} 4 \eta^2 \sum_{b=0}^{B-1} \alpha_b \sum_{u \in \mathcal{U}_b} \alpha_u ~ \sum_{\tau=\bar{t}_0}^{t-1} \mathbb{E} \left[ \left\Vert \nabla f_{u} (\mathbf{w}_{u}^\tau) - g_u (\mathbf{w}_u^{\tau}) \right\Vert^2 \right] - 4 \eta^2 \sum_{b=0}^{B-1} \alpha_b ~ \sum_{\tau=\bar{t}_0}^{t-1} \sum_{u \in \mathcal{U}_b} \alpha_{u}^2 \mathbb{E} \left[ \left\Vert  \left( \nabla f_{u} (\mathbf{w}_{u}^\tau) - g_{u} (\mathbf{w}_{u}^\tau) \right) \right\Vert^2 \right] \nonumber\\
     &\overset{(c)}{\leq} 4 \eta^2 \kappa_0^2 \sum_{b=0}^{B-1} \alpha_b \sum_{u \in \mathcal{U}_b} \alpha_u \cdot \sigma^2 - 4 \eta^2 \kappa_0^2 \sum_{b=0}^{B-1} \alpha_b \sum_{u \in \mathcal{U}_b} \alpha_{u}^2 \cdot \sigma^2 \nonumber\\
     &=4 \eta^2 \kappa_0^2 \sigma^2 - 4 \eta^2 \kappa_0^2 \sigma^2 \sum_{b=0}^{B-1} \alpha_b \sum_{u \in \mathcal{U}_b} \alpha_{u}^2,
 \end{align}
 where $(a)$ stems from the fact that $\sum_{i=0}^{I-1} \alpha_i \Vert \mathbf{a}_i - \bar{\mathbf{a}} \Vert^2 = \sum_{i=0}^{I-1} \alpha_i \Vert \mathbf{a}_i \Vert^2 - \Vert \bar{\mathbf{a}} \Vert^2$ if $\bar{\mathbf{a}} =\sum_{i=0}^{I-1} \alpha_i \mathbf{a}_i$, $0\leq \alpha_i \leq 1$ and $\sum_{i=0}^{I-1} \alpha_i=1$.
 In $(b)$, we use the independence of the stochastic gradients.
 In $(c)$, we use the fact that $[(t-1)-\bar{t}_0] \leq \kappa_0$ and Assumption \ref{AsumpUnbiased}.

 Now, we bound the second term as
 \begin{align}
 \label{secondTermLemma1}
     &4 \eta^2 \sum_{b=0}^{B-1} \alpha_b \sum_{u \in \mathcal{U}_b} \alpha_u ~ \mathbb{E} \Bigg[ \bigg\Vert \sum_{\tau=\bar{t}_0}^{t-1} \left( \sum_{u' \in \mathcal{U}_b} \alpha_{u'} \nabla f_{u'} (\mathbf{w}_{u'}^\tau) - \nabla f_{u} (\mathbf{w}_{u}^\tau) \right) \bigg] \bigg\Vert^2 \Bigg] \nonumber \\
     &\overset{(a)}{\leq} 4 \kappa_0 \eta^2 \sum_{b=0}^{B-1} \alpha_b \sum_{u \in \mathcal{U}_b} \alpha_u ~ \sum_{\tau=\bar{t}_0}^{t-1} \mathbb{E} \left[ \left\Vert \sum_{u' \in \mathcal{U}_b} \alpha_{u'} \nabla f_{u'} (\mathbf{w}_{u'}^\tau) - \nabla f_{u} (\mathbf{w}_{u}^\tau) \right\Vert^2 \right] \nonumber \\
     &\overset{(b)}{\leq} 4 \eta^2 \kappa_0^2 \sum_{b=0}^{B-1} \alpha_b \sum_{u \in \mathcal{U}_b} \alpha_u ~ \mathbb{E} \left[ \left\Vert \sum_{u' \in \mathcal{U}_b} \alpha_{u'} \nabla f_{u'} (\mathbf{w}_{u'}^t) - \sum_{u' \in \mathcal{U}_b} \alpha_{u'} \nabla f_{u'} (\bar{\mathbf{w}}_{b}^t) + \sum_{u' \in \mathcal{U}_b} \alpha_{u'} \nabla f_{u'} (\bar{\mathbf{w}}_{b}^t) - \nabla f_{u} (\bar{\mathbf{w}}_{b}^t) + \nabla f_{u} (\bar{\mathbf{w}}_{b}^t) - \nabla f_{u} (\mathbf{w}_{u}^t) \right\Vert^2 \right] \nonumber \\
     &\overset{(c)}{\leq} 12 \eta^2 \kappa_0^2 \sum_{b=0}^{B-1} \alpha_b \sum_{u \in \mathcal{U}_b} \alpha_u ~ \mathbb{E} \left[ \left\Vert \sum_{u' \in \mathcal{U}_b} \alpha_{u'} \nabla f_{u'} (\mathbf{w}_{u'}^t) - \sum_{u' \in \mathcal{U}_b} \alpha_{u'} \nabla f_{u'} (\bar{\mathbf{w}}_{b}^t) \right\Vert^2 \right] + \nonumber\\
     &\qquad 12 \eta^2 \kappa_0^2 \sum_{b=0}^{B-1} \alpha_b \sum_{u \in \mathcal{U}_b} \alpha_u ~ \mathbb{E} \left[ \left\Vert \sum_{u' \in \mathcal{U}_b} \alpha_{u'} \nabla f_{u'} (\bar{\mathbf{w}}_{b}^t) - \nabla f_{u} (\bar{\mathbf{w}}_{b}^t) \right\Vert^2\right] + 12 \eta^2 \kappa_0^2 \sum_{b=0}^{B-1} \alpha_b \sum_{u \in \mathcal{U}_b} \alpha_u ~ \mathbb{E} \left[ \left\Vert \nabla f_{u} (\bar{\mathbf{w}}_{b}^t) - \nabla f_{u} (\mathbf{w}_{u}^t) \right\Vert^2 \right] \nonumber \\
     &\overset{(d)}{\leq} 12 \eta^2 \kappa_0^2 \sum_{b=0}^{B-1} \alpha_b \sum_{u \in \mathcal{U}_b} \alpha_u ~ \mathbb{E} \left[ \left\Vert \nabla f_{u} (\mathbf{w}_{u}^t) - \nabla f_{u} (\bar{\mathbf{w}}_{b}^t) \right\Vert^2 \right] + 12 \eta^2 \kappa_0^2 \sum_{b=0}^{B-1} \alpha_b \sum_{u \in \mathcal{U}_b} \alpha_u ~ \mathbb{E} \left[ \left\Vert \nabla f_{b} (\bar{\mathbf{w}}_{b}^t) - \nabla f_{u} (\bar{\mathbf{w}}_{b}^t) \right\Vert^2\right] + \nonumber\\
     &\qquad 12 \eta^2 \kappa_0^2 \sum_{b=0}^{B-1} \alpha_b \sum_{u \in \mathcal{U}_b} \alpha_u ~ \mathbb{E} \left[ \left\Vert \nabla f_{u} (\bar{\mathbf{w}}_{b}^t) - \nabla f_{u} (\mathbf{w}_{u}^t) \right\Vert^2 \right] \nonumber \\
     &\overset{(e)}{\leq} 12 \eta^2 \kappa_0^2 \sum_{b=0}^{B-1} \alpha_b \sum_{u \in \mathcal{U}_b} \alpha_u ~ \beta^2 \mathbb{E} \left[ \left\Vert \mathbf{w}_{u}^t - \bar{\mathbf{w}}_{b}^t \right\Vert^2 \right] + 12 \eta^2 \kappa_0^2 \sum_{b=0}^{B-1} \alpha_b \sum_{u \in \mathcal{U}_b} \alpha_u \cdot \epsilon_0^2 + 12 \eta^2 \kappa_0^2 \sum_{b=0}^{B-1} \alpha_b \sum_{u \in \mathcal{U}_b} \alpha_u ~ \beta^2 \mathbb{E} \left[ \left\Vert \bar{\mathbf{w}}_{b}^t - \mathbf{w}_{u}^t \right\Vert^2 \right] \nonumber \\
     &= 12 \epsilon_0^2 \eta^2 \kappa_0^2 + 24 \beta^2 \eta^2 \kappa_0^2 \sum_{b=0}^{B-1} \alpha_b \sum_{u \in \mathcal{U}_b} \alpha_u \mathbb{E} \left[ \left\Vert \mathbf{w}_{u}^t - \bar{\mathbf{w}}_{b}^t \right\Vert^2 \right] 
 \end{align}
 where we use the fact that $[(t-1)-\bar{t}_0] \leq \kappa_0$ in $(a)$ and $(b)$.
 $(c)$ stems from $\Vert \sum_{i=1}^I \mathbf{a}_i \Vert^2 \leq I  \sum_{i=1}^I \Vert \mathbf{a}_i \Vert^2$.
 In $(d)$, we use the fact that $\sum_{u\in\mathcal{U}_b} \alpha_u f_u (\mathbf{w}) = f_b(\mathbf{w})$ and Jensen inequality. 
 Besides, $(e)$ is true due to Assumptions \ref{AsumpSmoothNess} and \ref{AsumpBoundedDivergence}.

 Now, plugging (\ref{firstTermLemma1}) and (\ref{secondTermLemma1}) into (\ref{lemma1MainEqn}), we get
 \begin{align}
     &2 \sum_{b=0}^{B-1} \alpha_b \sum_{u \in \mathcal{U}_b} \alpha_u ~ \mathbb{E} \left[ \left\Vert \bar{\mathbf{w}}_b^t - \mathbf{w}_{u}^t \right\Vert^2 \right] \nonumber\\
     &\leq 4 \eta^2 \kappa_0^2 \sigma^2 + 12 \epsilon_0^2 \eta^2 \kappa_0^2 - 4 \eta^2 \kappa_0^2 \sigma^2 \sum_{b=0}^{B-1} \alpha_b \sum_{u \in \mathcal{U}_b} \alpha_{u}^2 + 24 \beta^2 \eta^2 \kappa_0^2 \sum_{b=0}^{B-1} \alpha_b \sum_{u \in \mathcal{U}_b} \alpha_u \mathbb{E} \left[ \left\Vert \mathbf{w}_{u}^t - \bar{\mathbf{w}}_{b}^t \right\Vert^2 \right] .
 \end{align}
 Rearranging the terms and dividing both sides by $2$, we have
 \begin{align}
     &(1 - 12 \beta^2 \eta^2 \kappa_0^2 ) \sum_{b=0}^{B-1} \alpha_b \sum_{u \in \mathcal{U}_b} \alpha_u ~ \mathbb{E} \left[ \left\Vert \bar{\mathbf{w}}_b^t - \mathbf{w}_{u}^t \right\Vert^2 \right] \leq 2 \eta^2 \kappa_0^2 \sigma^2 + 6 \epsilon_0^2 \eta^2 \kappa_0^2 - 2 \eta^2 \kappa_0^2 \sigma^2 \sum_{b=0}^{B-1} \alpha_b \sum_{u \in \mathcal{U}_b} \alpha_{u}^2 . 
 \end{align}
 If $\eta \leq \frac{1}{2\sqrt{3} \beta \kappa_0}$, we have $0 < (1 - 12\beta^2\eta^2\kappa_0^2) <1$. Thus, we write
 \begin{align}
     &\sum_{b=0}^{B-1} \alpha_b \sum_{u \in \mathcal{U}_b} \alpha_u ~ \mathbb{E} \left[ \left\Vert \bar{\mathbf{w}}_b^t - \mathbf{w}_{u}^t \right\Vert^2 \right] \leq 2 \eta^2 \kappa_0^2 \sigma^2 + 6 \epsilon_0^2 \eta^2 \kappa_0^2 - 2 \eta^2 \kappa_0^2 \sigma^2 \sum_{b=0}^{B-1} \alpha_b \sum_{u \in \mathcal{U}_b} \alpha_{u}^2.
 \end{align}

 \end{proof}

 \begin{Lemma}
 \label{lemma2App}
 For the second term of (\ref{convExpectationTerm_1}), we have
 \begin{align}
     &2 \sum_{b=0}^{B-1} \alpha_b ~ \mathbb{E} \left[ \left\Vert \bar{\mathbf{w}}^t - \bar{\mathbf{w}}_b^t \right\Vert^2 \right] 
     \leq 20 \epsilon_1^2 \eta^2 \kappa_1^2 \kappa_0^2 + 80 \beta^2 \kappa_1^2 \sigma^2 \eta^4 \kappa_0^4 + 240 \beta^2 \epsilon_0^2 \kappa_1^2 \eta^4 \kappa_0^4 + 4 \kappa_1 \kappa_0 \eta^2 \sigma^2 \sum_{b=0}^{B-1} \alpha_b \sum_{u \in \mathcal{U}_b} \alpha_{u}^2 - \nonumber\\
     &\qquad\qquad 4 \kappa_1 \kappa_0 \eta^2 \sigma^2 \sum_{b=0}^{B-1} \alpha_{b}^2 \sum_{u\in \mathcal{U}_{b}} \alpha_u^2 - 80 \beta^2 \kappa_1^2 \sigma^2 \eta^4 \kappa_0^4 \sum_{b=0}^{B-1} \alpha_b \sum_{u \in \mathcal{U}_b} \alpha_{u}^2 + 40 \beta^2 \eta^2 \kappa_1^2\kappa_0^2 \sum_{b=0}^{B-1} \alpha_b \mathbb{E} \left[ \left\Vert \bar{\mathbf{w}}^{t} - \bar{\mathbf{w}}_b^{t} \right \Vert^2 \right].
 \end{align}
 Besides, if $\eta \leq \frac{1}{2\sqrt{5} \beta \kappa_1\kappa_0}$, we have 
 \begin{align}
     \sum_{b=0}^{B-1} \alpha_b ~ \mathbb{E} \left[ \left\Vert \bar{\mathbf{w}}^t - \bar{\mathbf{w}}_b^t \right\Vert^2 \right] 
     \leq& 10 \epsilon_1^2 \eta^2 \kappa_1^2 \kappa_0^2 + 40 \beta^2 \kappa_1^2 \sigma^2 \eta^4 \kappa_0^4 + 120 \beta^2 \epsilon_0^2 \kappa_1^2 \eta^4 \kappa_0^4 + 2 \kappa_1 \kappa_0 \eta^2 \sigma^2 \sum_{b=0}^{B-1} \alpha_b \sum_{u \in \mathcal{U}_b} \alpha_{u}^2 - \nonumber\\
     &2 \kappa_1 \kappa_0 \eta^2 \sigma^2 \sum_{b=0}^{B-1} \alpha_{b}^2 \sum_{u\in \mathcal{U}_{b}} \alpha_u^2 - 40 \beta^2 \kappa_1^2 \sigma^2 \eta^4 \kappa_0^4 \sum_{b=0}^{B-1} \alpha_b \sum_{u \in \mathcal{U}_b} \alpha_{u}^2.
 \end{align}
 \end{Lemma}

 \begin{proof}
 Suppose $\bar{t}_1 \coloneqq t_2\kappa_1\kappa_0$ and $t = \bar{t}_1 + t_1\kappa_0 + t_0$, where $1 \leq t_1\kappa_0 + t_0 \leq (\kappa_1\kappa_0 - 1)$. 
 Then we have 
 \begin{align}
 \label{lemma2MainEqn}
     &2 \sum_{b=0}^{B-1} \alpha_b ~ \mathbb{E} \left[ \left\Vert \bar{\mathbf{w}}^t - \bar{\mathbf{w}}_b^t \right\Vert^2 \right] \nonumber\\
     &= 2 \sum_{b=0}^{B-1} \alpha_b ~ \mathbb{E} \left[ \left\Vert \mathbf{w}^{\bar{t}_1} - \eta \sum_{\tau=\bar{t}_1}^{t-1} \sum_{b'=0}^{B-1} \alpha_{b'} \sum_{u\in \mathcal{U}_{b'}} \alpha_u  g_u (\mathbf{w}_u^{\tau}) - \mathbf{w}_b^{\bar{t}_1} + \eta\sum_{u \in \mathcal{U}_b} \alpha_{u} \sum_{\tau=\bar{t}_1}^{t-1} g_{u} (\mathbf{w}_{u}^\tau)  \right\Vert^2 \right] \nonumber\\
     &\overset{(a)}{=} 2 \eta^2 \sum_{b=0}^{B-1} \alpha_b ~ \mathbb{E} \left[ \left\Vert \sum_{\tau=\bar{t}_1}^{t-1} \left[ \sum_{u \in \mathcal{U}_b} \alpha_{u} g_{u} (\mathbf{w}_{u}^\tau) - \sum_{b'=0}^{B-1} \alpha_{b'} \sum_{u\in \mathcal{U}_{b'}} \alpha_u  g_u (\mathbf{w}_u^{\tau}) \right] \right\Vert^2 \right] \nonumber\\
     &= 2 \eta^2 \sum_{b=0}^{B-1} \alpha_b ~ \mathbb{E} \Bigg[ \Bigg\Vert \sum_{\tau=\bar{t}_1}^{t-1} \Bigg[ \sum_{u \in \mathcal{U}_b} \alpha_{u} g_{u} (\mathbf{w}_{u}^\tau) - \sum_{u \in \mathcal{U}_b} \alpha_{u} \nabla f_{u} (\mathbf{w}_{u}^\tau) + \sum_{u \in \mathcal{U}_b} \alpha_{u} \nabla f_{u} (\mathbf{w}_{u}^\tau) - \sum_{b'=0}^{B-1} \alpha_{b'} \sum_{u\in \mathcal{U}_{b'}} \alpha_u \nabla f_u (\mathbf{w}_u^{\tau}) + \nonumber\\
     & \qquad \qquad \sum_{b'=0}^{B-1} \alpha_{b'} \sum_{u\in \mathcal{U}_{b'}} \alpha_u  \nabla f_u (\mathbf{w}_u^{\tau}) - \sum_{b'=0}^{B-1} \alpha_{b'} \sum_{u\in \mathcal{U}_{b'}} \alpha_u  g_u (\mathbf{w}_u^{\tau}) \Bigg] \Bigg\Vert^2 \Bigg] \nonumber\\
     &\overset{(b)}{\leq} 4 \eta^2 \sum_{b=0}^{B-1} \alpha_b ~ \mathbb{E} \Bigg[ \Bigg\Vert \sum_{\tau=\bar{t}_1}^{t-1} \Bigg[ \sum_{u \in \mathcal{U}_b} \alpha_{u} \left( g_{u} (\mathbf{w}_{u}^\tau) - \nabla f_{u} (\mathbf{w}_{u}^\tau) \right) + \sum_{b'=0}^{B-1} \alpha_{b'} \sum_{u\in \mathcal{U}_{b'}} \alpha_u \left( \nabla f_u (\mathbf{w}_u^{\tau}) - g_u (\mathbf{w}_u^{\tau}) \right) \Bigg] \Bigg\Vert^2 \Bigg] \nonumber\\
     &\qquad\qquad 4 \eta^2 \sum_{b=0}^{B-1} \alpha_b ~ \mathbb{E} \Bigg[ \Bigg\Vert \sum_{\tau=\bar{t}_1}^{t-1} \Bigg[ \sum_{u \in \mathcal{U}_b} \alpha_{u} \nabla f_{u} (\mathbf{w}_{u}^\tau) - \sum_{b'=0}^{B-1} \alpha_{b'} \sum_{u\in \mathcal{U}_{b'}} \alpha_u \nabla f_u (\mathbf{w}_u^{\tau}) \Bigg] \Bigg\Vert^2 \Bigg],
 \end{align}
 where $(a)$ stems from the fact that $\mathbf{w}_u^{\bar{t}_1} \gets \mathbf{w}_b^{\bar{t}_1}$ at $\bar{t}_1$ and the independence of the stochastic gradients.

 Now, we simplify the first term of (\ref{lemma2MainEqn}) as follows:
 \begin{align}
 \label{lemma2MainEqnFirstTerm}
     &4 \eta^2 \sum_{b=0}^{B-1} \alpha_b ~ \mathbb{E} \Bigg[ \Bigg\Vert \sum_{\tau=\bar{t}_1}^{t-1} \Bigg[ \sum_{u \in \mathcal{U}_b} \alpha_{u} \left( g_{u} (\mathbf{w}_{u}^\tau) - \nabla f_{u} (\mathbf{w}_{u}^\tau) \right) + \sum_{b'=0}^{B-1} \alpha_{b'} \sum_{u\in \mathcal{U}_{b'}} \alpha_u \left( \nabla f_u (\mathbf{w}_u^{\tau}) - g_u (\mathbf{w}_u^{\tau}) \right) \Bigg] \Bigg\Vert^2 \Bigg] \nonumber\\
     &=4 \eta^2 \sum_{b=0}^{B-1} \alpha_b ~ \mathbb{E} \Bigg[ \Bigg\Vert \sum_{\tau=\bar{t}_1}^{t-1} \Bigg[ \sum_{u \in \mathcal{U}_b} \alpha_{u} \left( g_{u} (\mathbf{w}_{u}^\tau) - \nabla f_{u} (\mathbf{w}_{u}^\tau) \right) - \sum_{b'=0}^{B-1} \alpha_{b'} \sum_{u\in \mathcal{U}_{b'}} \alpha_u \left( g_u (\mathbf{w}_u^{\tau}) - \nabla f_u (\mathbf{w}_u^{\tau}) \right) \Bigg] \Bigg\Vert^2 \Bigg] \nonumber\\
     &\overset{(a)}{=} 4 \eta^2 \sum_{b=0}^{B-1} \alpha_b ~ \mathbb{E} \Bigg[ \Bigg\Vert \sum_{\tau=\bar{t}_1}^{t-1} \Bigg[ \sum_{u \in \mathcal{U}_b} \alpha_{u} \left( g_{u} (\mathbf{w}_{u}^\tau) - \nabla f_{u} (\mathbf{w}_{u}^\tau) \right) \Bigg] \Bigg\Vert^2 \Bigg] - 4 \eta^2 ~ \mathbb{E} \Bigg[ \Bigg\Vert \sum_{\tau=\bar{t}_1}^{t-1} \Bigg[ \sum_{b'=0}^{B-1} \alpha_{b'} \sum_{u\in \mathcal{U}_{b'}} \alpha_u \left( g_u (\mathbf{w}_u^{\tau}) - \nabla f_u (\mathbf{w}_u^{\tau}) \right) \Bigg] \Bigg\Vert^2 \Bigg] \nonumber\\
     &\overset{(b)}{=} 4 \eta^2 \sum_{b=0}^{B-1} \alpha_b ~ \sum_{\tau=\bar{t}_1}^{t-1} \sum_{u \in \mathcal{U}_b} \alpha_{u}^2 \mathbb{E} \left[ \left\Vert  g_{u} (\mathbf{w}_{u}^\tau) - \nabla f_{u} (\mathbf{w}_{u}^\tau) \right\Vert^2 \right] - 4 \eta^2 ~\sum_{\tau=\bar{t}_1}^{t-1} \sum_{b=0}^{B-1} \alpha_{b}^2 \sum_{u\in \mathcal{U}_{b}} \alpha_u^2 \mathbb{E} \left[ \left\Vert g_u (\mathbf{w}_u^{\tau}) - \nabla f_u (\mathbf{w}_u^{\tau}) \right\Vert^2 \right] \nonumber\\
     &\overset{(c)}{\leq} 4 \eta^2 \sum_{b=0}^{B-1} \alpha_b ~ \kappa_1 \kappa_0 \sum_{u \in \mathcal{U}_b} \alpha_{u}^2 \sigma^2 - 4 \kappa_1 \kappa_0 \eta^2  \sum_{b=0}^{B-1} \alpha_{b}^2 \sum_{u\in \mathcal{U}_{b}} \alpha_u^2 \sigma^2 \nonumber\\
     &=4 \kappa_1 \kappa_0 \eta^2 \sigma^2 \sum_{b=0}^{B-1} \alpha_b \sum_{u \in \mathcal{U}_b} \alpha_{u}^2 - 4 \kappa_1 \kappa_0 \eta^2 \sigma^2 \sum_{b=0}^{B-1} \alpha_{b}^2 \sum_{u\in \mathcal{U}_{b}} \alpha_u^2,
 \end{align}
 where $(a)$ stems from the fact that $\sum_{i=0}^{I-1} \alpha_i \Vert \mathbf{a}_i - \bar{\mathbf{a}} \Vert^2 = \sum_{i=0}^{I-1} \alpha_i \Vert \mathbf{a}_i \Vert^2 - \Vert \bar{\mathbf{a}} \Vert^2$ if $\bar{\mathbf{a}} =\sum_{i=0}^{I-1} \alpha_i \mathbf{a}_i$, $0\leq \alpha_i \leq 1$ and $\sum_{i=0}^{I-1} \alpha_i=1$.
 In $(b)$, we use the independence of the stochastic gradients.
 In $(c)$, we use the fact that $[(t-1)-\bar{t}_1] \leq \kappa_1\kappa_0$ and Assumption \ref{AsumpUnbiased}.

 Similarly, we simplify the second term of (\ref{lemma2MainEqn}) as 
 \begin{align}
 \label{lemma2MainEqnSecondTerm}
     &4 \eta^2 \sum_{b=0}^{B-1} \alpha_b ~ \mathbb{E} \Bigg[ \Bigg\Vert \sum_{\tau=\bar{t}_1}^{t-1} \Bigg[ \sum_{u \in \mathcal{U}_b} \alpha_{u} \nabla f_{u} (\mathbf{w}_{u}^\tau) - \sum_{b'=0}^{B-1} \alpha_{b'} \sum_{u\in \mathcal{U}_{b'}} \alpha_u \nabla f_u (\mathbf{w}_u^{\tau}) \Bigg] \Bigg\Vert^2 \Bigg] \nonumber\\
     &\overset{(a)}{\leq} 4 \kappa_0\kappa_1 \eta^2 \sum_{b=0}^{B-1} \alpha_b ~ \sum_{\tau=\bar{t}_1}^{t-1} \mathbb{E} \left[ \left\Vert \sum_{u \in \mathcal{U}_b} \alpha_{u} \nabla f_{u} (\mathbf{w}_{u}^\tau) - \sum_{b'=0}^{B-1} \alpha_{b'} \sum_{u\in \mathcal{U}_{b'}} \alpha_u \nabla f_u (\mathbf{w}_u^{\tau}) \right \Vert^2 \right] \nonumber\\
     &\overset{(b)}{\leq} 4 \eta^2 \kappa_1^2\kappa_0^2 \sum_{b=0}^{B-1} \alpha_b ~ \mathbb{E} \Bigg[ \Bigg\Vert \sum_{u \in \mathcal{U}_b} \alpha_{u} \nabla f_{u} (\mathbf{w}_{u}^t) - \sum_{u \in \mathcal{U}_b} \alpha_{u} \nabla f_{u} (\bar{\mathbf{w}}_{b}^t) + \sum_{u \in \mathcal{U}_b} \alpha_{u} \nabla f_{u} (\bar{\mathbf{w}}_{b}^t) - \sum_{u \in \mathcal{U}_b} \alpha_{u} \nabla f_{u} (\bar{\mathbf{w}}^t) + \nonumber\\
     &\qquad\qquad \sum_{u \in \mathcal{U}_b} \alpha_{u} \nabla f_{u} (\bar{\mathbf{w}}^t) - \sum_{b'=0}^{B-1} \alpha_{b'} \sum_{u\in \mathcal{U}_{b'}} \alpha_u \nabla f_u (\bar{\mathbf{w}}^{t}) +  \sum_{b'=0}^{B-1} \alpha_{b'} \sum_{u\in \mathcal{U}_{b'}} \alpha_u \nabla f_u (\bar{\mathbf{w}}^{t}) - \sum_{b'=0}^{B-1} \alpha_{b'} \sum_{u\in \mathcal{U}_{b'}} \alpha_u \nabla f_u (\bar{\mathbf{w}}_b^{t}) + \nonumber\\
     &\qquad\qquad \sum_{b'=0}^{B-1} \alpha_{b'} \sum_{u\in \mathcal{U}_{b'}} \alpha_u \nabla f_u (\bar{\mathbf{w}}_b^{t}) - \sum_{b'=0}^{B-1} \alpha_{b'} \sum_{u\in \mathcal{U}_{b'}} \alpha_u \nabla f_u (\mathbf{w}_u^{t}) \Bigg \Vert^2 \Bigg] \nonumber\\
     &\overset{(c)}{\leq} 20 \eta^2 \kappa_1^2\kappa_0^2 \sum_{b=0}^{B-1} \alpha_b ~ \mathbb{E} \Bigg[ \Bigg\Vert \sum_{u \in \mathcal{U}_b} \alpha_{u} \nabla f_{u} (\mathbf{w}_{u}^t) - \sum_{u \in \mathcal{U}_b} \alpha_{u} \nabla f_{u} (\bar{\mathbf{w}}_{b}^t) \Bigg \Vert^2 \Bigg] + 20 \eta^2 \kappa_1^2\kappa_0^2 \sum_{b=0}^{B-1} \alpha_b ~ \mathbb{E} \Bigg[ \Bigg\Vert  \sum_{u \in \mathcal{U}_b} \alpha_{u} \nabla f_{u} (\bar{\mathbf{w}}_{b}^t) - \sum_{u \in \mathcal{U}_b} \alpha_{u} \nabla f_{u} (\bar{\mathbf{w}}^t) \Bigg \Vert^2 \Bigg] \nonumber\\
     &\qquad \qquad 20 \eta^2 \kappa_1^2\kappa_0^2 \sum_{b=0}^{B-1} \alpha_b ~ \mathbb{E} \Bigg[ \Bigg\Vert \sum_{u \in \mathcal{U}_b} \alpha_{u} \nabla f_{u} (\bar{\mathbf{w}}^t) - \sum_{b'=0}^{B-1} \alpha_{b'} \sum_{u\in \mathcal{U}_{b'}} \alpha_u \nabla f_u (\bar{\mathbf{w}}^{t}) \Bigg \Vert^2 \Bigg] + \nonumber\\
     &\qquad \qquad 20 \eta^2 \kappa_1^2\kappa_0^2 \sum_{b=0}^{B-1} \alpha_b ~ \mathbb{E} \Bigg[ \Bigg\Vert \sum_{b'=0}^{B-1} \alpha_{b'} \sum_{u\in \mathcal{U}_{b'}} \alpha_u \nabla f_u (\bar{\mathbf{w}}^{t}) - \sum_{b'=0}^{B-1} \alpha_{b'} \sum_{u\in \mathcal{U}_{b'}} \alpha_u \nabla f_u (\bar{\mathbf{w}}_b^{t}) \Bigg \Vert^2 \Bigg] + \nonumber\\
     &\qquad \qquad 20 \eta^2 \kappa_1^2\kappa_0^2 \sum_{b=0}^{B-1} \alpha_b ~ \mathbb{E} \Bigg[ \Bigg\Vert \sum_{b'=0}^{B-1} \alpha_{b'} \sum_{u\in \mathcal{U}_{b'}} \alpha_u \nabla f_u (\bar{\mathbf{w}}_b^{t}) - \sum_{b'=0}^{B-1} \alpha_{b'} \sum_{u\in \mathcal{U}_{b'}} \alpha_u \nabla f_u (\mathbf{w}_u^{t}) \Bigg \Vert^2 \Bigg] \nonumber\\
     &\overset{(d)}{\leq} 20 \eta^2 \kappa_1^2\kappa_0^2 \sum_{b=0}^{B-1} \alpha_b \sum_{u \in \mathcal{U}_b} \alpha_{u} \mathbb{E} \left[ \left\Vert  \nabla f_{u} (\mathbf{w}_{u}^t) - \nabla f_{u} (\bar{\mathbf{w}}_{b}^t) \right \Vert^2 \right] + 20 \eta^2 \kappa_1^2\kappa_0^2 \sum_{b=0}^{B-1} \alpha_b ~ \mathbb{E} \left[ \left\Vert  \nabla f_{b} (\bar{\mathbf{w}}_{b}^t) -  \nabla f_{b} (\bar{\mathbf{w}}^t) \right \Vert^2 \right] \nonumber\\
     &\qquad \qquad 20 \eta^2 \kappa_1^2\kappa_0^2 \sum_{b=0}^{B-1} \alpha_b ~ \mathbb{E} \left[ \left\Vert  \nabla f_b (\bar{\mathbf{w}}^t) - \nabla f (\bar{\mathbf{w}}^{t}) \right \Vert^2 \right] + 20 \eta^2 \kappa_1^2\kappa_0^2 \sum_{b=0}^{B-1} \alpha_b ~ \sum_{u\in \mathcal{U}_{b}} \alpha_u \mathbb{E} \left[ \left\Vert  \nabla f_u (\bar{\mathbf{w}}^{t}) - \nabla f_u (\bar{\mathbf{w}}_b^{t}) \right \Vert^2 \right] + \nonumber\\
     &\qquad \qquad 20 \eta^2 \kappa_1^2\kappa_0^2 \sum_{b=0}^{B-1} \alpha_b ~ \sum_{u\in \mathcal{U}_{b}} \alpha_u  \mathbb{E} \left[ \left\Vert \nabla f_u (\bar{\mathbf{w}}_b^{t}) - \nabla f_u (\mathbf{w}_u^{t}) \right \Vert^2 \right] \nonumber\\
     &\overset{(e)}{\leq} 20 \eta^2 \kappa_1^2\kappa_0^2 \sum_{b=0}^{B-1} \alpha_b \sum_{u \in \mathcal{U}_b} \alpha_{u} \cdot \beta^2 \mathbb{E} \left[ \left\Vert  \mathbf{w}_{u}^t - \bar{\mathbf{w}}_{b}^t \right \Vert^2 \right] + 20 \eta^2 \kappa_1^2\kappa_0^2 \sum_{b=0}^{B-1} \alpha_b \cdot \beta^2  \mathbb{E} \left[ \left\Vert \bar{\mathbf{w}}_{b}^t - \bar{\mathbf{w}}^t \right \Vert^2 \right] + 20 \eta^2 \kappa_1^2\kappa_0^2 \sum_{b=0}^{B-1} \alpha_b \cdot \epsilon_1^2 + \nonumber\\
     &\qquad \qquad 20 \eta^2 \kappa_1^2\kappa_0^2 \sum_{b=0}^{B-1} \alpha_b \sum_{u\in \mathcal{U}_{b}} \alpha_u \cdot \beta^2 \mathbb{E} \left[ \left\Vert \bar{\mathbf{w}}^{t} - \bar{\mathbf{w}}_b^{t} \right \Vert^2 \right] + 20 \eta^2 \kappa_1^2\kappa_0^2 \sum_{b=0}^{B-1} \alpha_b \sum_{u\in \mathcal{U}_{b}} \alpha_u \cdot \beta^2 \mathbb{E} \left[ \left\Vert \bar{\mathbf{w}}_b^{t} - \mathbf{w}_u^{t} \right \Vert^2 \right] \nonumber\\
     &= 20 \epsilon_1^2 \eta^2 \kappa_1^2 \kappa_0^2 + 40 \beta^2 \eta^2 \kappa_1^2\kappa_0^2 \sum_{b=0}^{B-1} \alpha_b \sum_{u\in \mathcal{U}_{b}} \alpha_u \mathbb{E} \left[ \left\Vert \bar{\mathbf{w}}_b^{t} - \mathbf{w}_u^{t} \right \Vert^2 \right] + 40 \beta^2 \eta^2 \kappa_1^2\kappa_0^2 \sum_{b=0}^{B-1} \alpha_b \mathbb{E} \left[ \left\Vert \bar{\mathbf{w}}^{t} - \bar{\mathbf{w}}_b^{t} \right \Vert^2 \right] \nonumber\\
     &\overset{(f)}{\leq} 20 \epsilon_1^2 \eta^2 \kappa_1^2 \kappa_0^2 + 40 \beta^2 \eta^2 \kappa_1^2\kappa_0^2 \times \left( 2 \eta^2 \kappa_0^2 \sigma^2 + 6 \epsilon_0^2 \eta^2 \kappa_0^2 - 2 \eta^2 \kappa_0^2 \sigma^2 \sum_{b=0}^{B-1} \alpha_b \sum_{u \in \mathcal{U}_b} \alpha_{u}^2\right) + 40 \beta^2 \eta^2 \kappa_1^2\kappa_0^2 \sum_{b=0}^{B-1} \alpha_b \mathbb{E} \left[ \left\Vert \bar{\mathbf{w}}^{t} - \bar{\mathbf{w}}_b^{t} \right \Vert^2 \right] \nonumber\\
     &=20 \epsilon_1^2 \eta^2 \kappa_1^2 \kappa_0^2 + 80 \beta^2 \kappa_1^2 \sigma^2 \eta^4 \kappa_0^4 + 240 \beta^2 \epsilon_0^2 \kappa_1^2 \eta^4 \kappa_0^4 - 80 \beta^2 \kappa_1^2 \sigma^2 \eta^4 \kappa_0^4 \sum_{b=0}^{B-1} \alpha_b \sum_{u \in \mathcal{U}_b} \alpha_{u}^2 + 40 \beta^2 \eta^2 \kappa_1^2\kappa_0^2 \sum_{b=0}^{B-1} \alpha_b \mathbb{E} \left[ \left\Vert \bar{\mathbf{w}}^{t} - \bar{\mathbf{w}}_b^{t} \right \Vert^2 \right]
 \end{align}
 where we use the fact that $[(t-1)-\bar{t}_1] \leq \kappa_1\kappa_0$ in $(a)$ and $(b)$.
 Besides, $(c)$ stems from $\Vert \sum_{i=1}^I \mathbf{a}_i \Vert^2 \leq I  \sum_{i=1}^I \Vert \mathbf{a}_i \Vert^2$.
 In $(d)$, we use Jensen inequality and the definitions of the loss functions.
 Furthermore, $(e)$ comes from our Assumptions \ref{AsumpSmoothNess} and \ref{AsumpBoundedDivergence}.
 Finally, $(f)$ is derived using Lemma \ref{lemma1App}.

 Now, plugging (\ref{lemma2MainEqnFirstTerm}) and (\ref{lemma2MainEqnSecondTerm}) into (\ref{lemma2MainEqn}), we get
 \begin{align}
     &2 \sum_{b=0}^{B-1} \alpha_b ~ \mathbb{E} \left[ \left\Vert \bar{\mathbf{w}}^t - \bar{\mathbf{w}}_b^t \right\Vert^2 \right] \nonumber\\
     &\leq 4 \kappa_1 \kappa_0 \eta^2 \sigma^2 \sum_{b=0}^{B-1} \alpha_b \sum_{u \in \mathcal{U}_b} \alpha_{u}^2 - 4 \kappa_1 \kappa_0 \eta^2 \sigma^2 \sum_{b=0}^{B-1} \alpha_{b}^2 \sum_{u\in \mathcal{U}_{b}} \alpha_u^2 + 20 \epsilon_1^2 \eta^2 \kappa_1^2 \kappa_0^2 + 80 \beta^2 \kappa_1^2 \sigma^2 \eta^4 \kappa_0^4 + 240 \beta^2 \epsilon_0^2 \kappa_1^2 \eta^4 \kappa_0^4 \nonumber\\
     &\qquad \qquad - 80 \beta^2 \kappa_1^2 \sigma^2 \eta^4 \kappa_0^4 \sum_{b=0}^{B-1} \alpha_b \sum_{u \in \mathcal{U}_b} \alpha_{u}^2 + 40 \beta^2 \eta^2 \kappa_1^2\kappa_0^2 \sum_{b=0}^{B-1} \alpha_b \mathbb{E} \left[ \left\Vert \bar{\mathbf{w}}^{t} - \bar{\mathbf{w}}_b^{t} \right \Vert^2 \right] \nonumber\\
     &= 20 \epsilon_1^2 \eta^2 \kappa_1^2 \kappa_0^2 + 80 \beta^2 \kappa_1^2 \sigma^2 \eta^4 \kappa_0^4 + 240 \beta^2 \epsilon_0^2 \kappa_1^2 \eta^4 \kappa_0^4 + 4 \kappa_1 \kappa_0 \eta^2 \sigma^2 \sum_{b=0}^{B-1} \alpha_b \sum_{u \in \mathcal{U}_b} \alpha_{u}^2 - 4 \kappa_1 \kappa_0 \eta^2 \sigma^2 \sum_{b=0}^{B-1} \alpha_{b}^2 \sum_{u\in \mathcal{U}_{b}} \alpha_u^2 \nonumber\\
     &\qquad\qquad - 80 \beta^2 \kappa_1^2 \sigma^2 \eta^4 \kappa_0^4 \sum_{b=0}^{B-1} \alpha_b \sum_{u \in \mathcal{U}_b} \alpha_{u}^2 + 40 \beta^2 \eta^2 \kappa_1^2\kappa_0^2 \sum_{b=0}^{B-1} \alpha_b \mathbb{E} \left[ \left\Vert \bar{\mathbf{w}}^{t} - \bar{\mathbf{w}}_b^{t} \right \Vert^2 \right] 
 \end{align}
 Rearranging the terms and dividing both sides by $2$, we get
 \begin{align}
     &(1 - 20 \beta^2 \eta^2 \kappa_1^2\kappa_0^2) \sum_{b=0}^{B-1} \alpha_b ~ \mathbb{E} \left[ \left\Vert \bar{\mathbf{w}}^t - \bar{\mathbf{w}}_b^t \right\Vert^2 \right] \nonumber\\
     &\leq 10 \epsilon_1^2 \eta^2 \kappa_1^2 \kappa_0^2 + 40 \beta^2 \kappa_1^2 \sigma^2 \eta^4 \kappa_0^4 + 120 \beta^2 \epsilon_0^2 \kappa_1^2 \eta^4 \kappa_0^4 + 2 \kappa_1 \kappa_0 \eta^2 \sigma^2 \sum_{b=0}^{B-1} \alpha_b \sum_{u \in \mathcal{U}_b} \alpha_{u}^2 - 2 \kappa_1 \kappa_0 \eta^2 \sigma^2 \sum_{b=0}^{B-1} \alpha_{b}^2 \sum_{u\in \mathcal{U}_{b}} \alpha_u^2 \nonumber\\
     &\qquad\qquad - 40 \beta^2 \kappa_1^2 \sigma^2 \eta^4 \kappa_0^4 \sum_{b=0}^{B-1} \alpha_b \sum_{u \in \mathcal{U}_b} \alpha_{u}^2
 \end{align}
 If $\eta \leq \frac{1}{2\sqrt{5} \beta \kappa_1\kappa_0}$, we have $0 < (1 - 20\beta^2\eta^2\kappa_1^2\kappa_0^2) <1$. Thus, we write
 \begin{align}
     \sum_{b=0}^{B-1} \alpha_b ~ \mathbb{E} \left[ \left\Vert \bar{\mathbf{w}}^t - \bar{\mathbf{w}}_b^t \right\Vert^2 \right] 
     \leq& 10 \epsilon_1^2 \eta^2 \kappa_1^2 \kappa_0^2 + 40 \beta^2 \kappa_1^2 \sigma^2 \eta^4 \kappa_0^4 + 120 \beta^2 \epsilon_0^2 \kappa_1^2 \eta^4 \kappa_0^4 + 2 \kappa_1 \kappa_0 \eta^2 \sigma^2 \sum_{b=0}^{B-1} \alpha_b \sum_{u \in \mathcal{U}_b} \alpha_{u}^2 - \nonumber\\
     &2 \kappa_1 \kappa_0 \eta^2 \sigma^2 \sum_{b=0}^{B-1} \alpha_{b}^2 \sum_{u\in \mathcal{U}_{b}} \alpha_u^2 - 40 \beta^2 \kappa_1^2 \sigma^2 \eta^4 \kappa_0^4 \sum_{b=0}^{B-1} \alpha_b \sum_{u \in \mathcal{U}_b} \alpha_{u}^2
 \end{align}

 \end{proof}

\end{document}